\documentclass[journal]{IEEEtran}
\usepackage{epsf}
\usepackage{euscript}
\usepackage{upgreek}
\usepackage{amsfonts}
\usepackage{amsmath}
\usepackage{amssymb}
\usepackage{graphicx}
\usepackage{subfigure}
\usepackage{amsmath}
\usepackage{eepic}
\usepackage{epic}
\usepackage{xcolor}

\newtheorem{theorem}{Theorem}[section]
\newtheorem{remark}{Remark}[section]

\newtheorem{proposition}[theorem]{Proposition}

\newcommand*{\qed}{\hfill\ensuremath{\blacksquare}}%

\def\R{{\Bbb R}}

\newcommand{\ndof}[0]{n}
\newcommand{\al}[1]{\begin{align} #1 \end{align}}
\newcommand{\nn}{\nonumber}

\begin{document}

\title{Derivative-free online learning of  inverse dynamics  models}

\author{Diego Romeres$^\star$, Mattia Zorzi$^\dagger$,  Raffaello Camoriano$^{\ddagger \diamond}$, Silvio Traversaro$^\diamond$ and  Alessandro~Chiuso$^\dagger$ \thanks{This work has been  supported by the FIRB project ``Learning meets time'' (RBFR12M3AC) funded by MIUR.}
\thanks{$^\star$ Mitsubishi Electric Research Labs (MERL), Cambridge, MA, USA (e-mail: romeres@merl.com). Part of the work has been done at $^\dagger$}
\thanks{$^\dagger$ Dept. of Information  Engineering, University of Padova, Via Gradenigo 6/b, 35131, Padova, Italy (e-mail: \{\tt \small zorzimat,chiuso\}@dei.unipd.it)}
%
%
\thanks{$^{\ddagger}$ LCSL - IIT@MIT, Massachusetts Institute of Technology, Cambridge, MA 02139, USA}
\thanks{$^{\diamond}$  iCub Facility, Istituto Italiano di Tecnologia, Via Morego 30, Genoa 16163, Italy (e-mail:  \{\tt \small raffaello.camoriano,silvio.traversaro\}@iit.it)}
}
\markboth{}{Romeres et al.: Derivative-free online learning of  inverse dynamics models.}

\maketitle

\begin{abstract}
This paper discusses online algorithms for inverse dynamics modelling in robotics. Several model classes including rigid body dynamics (RBD)  models, data-driven models and semiparametric models (which are a combination of the previous two classes)  are placed in a common framework. 
While model classes used in the literature typically exploit joint velocities and accelerations, which need to be approximated resorting to numerical differentiation schemes, in this paper a new ``derivative-free'' framework is proposed that does not require this preprocessing step. 
An extensive experimental study with real data from the right arm of the iCub robot is presented, comparing different model classes and estimation procedures, showing that the proposed ``derivative-free'' methods outperform existing methodologies.  
\end{abstract}


\section{Introduction} 
\label{sec: intro}

Robotic platforms, such as industrial and service robots,  are becoming more and more popular and are prospected to pervade our everyday life in the near future. A key requirement for such systems is that they should be able to \emph{safely} interact with the environment, possibly including humans,  while \emph{performing robustly and efficiently} certain assigned
tasks;  in order to do so, they have to adapt their behavior to variable conditions.

Among other technological challenges, this requires new mathematical tools for data-driven modeling while accounting for inherent uncertainty and
changing conditions (\emph{learning}). While accomplishing user assigned tasks, the robotic system  acquires new data which needs to be processed \emph{online} to adapt these models.

In this paper we shall be concerned with \emph{online learning} of so called \emph{inverse dynamics models}: these models have joint trajectories as inputs and joint torques as outputs. They are  widely used  for  model-based control  in robotic applications  to improve the tracking performances leading to high accuracy and low control gains \cite{craig2005introduction, nguyen2009model}, \cite{KARIMI,fliess2009model}; see also the survey \cite{Nguyen-Tuong2011} for a comprehensive overview. Inverse dynamics models are also very important for detection and estimation of contact forces, see for instance \cite{IvaldiICRA2015}.

Typically, inverse dynamics models are obtained from  first principles, using the physics of rigid body dynamics (RBD), \cite{siciliano2010robotics}. This results in 
a parametric model that depends on some physical parameters (masses, lengths, etc.). Unfortunately, accurate knowledge of these physical parameters may often not be available; for this reason
an RBD  model should be estimated from data. The main advantage of the parametric approach is that in principle it provides a global  relationship between inputs and outputs. However, the parametric model strongly relies on  several assumptions and may be rather inaccurate when these assumptions are not perfectly satisfied (frictions, nonlinearities, non perfectly rigid links, etc.),  \cite{hollerbach2008model,siciliano2010robotics}.

Alternatively,  a nonparametric black-box model can be estimated from experimental data using machine learning techniques such as Gaussian Process regression \cite{Rasmussen}. The nonparametric framework has the advantage of not requiring unrealistic assumptions, but it comes at the price of being local in nature: the model can only be expected to reflect the systems dynamics  in a ``neighbourhood'' of the trajectories already seen during the learning phase.  In the context of system identification for linear dynamical systems a comparison between the capability of parametric and nonparametric approaches has been presented in \cite{prando2015classical}.

To exploit the advantages of both estimation techniques, semiparametric models have been recently introduced as a combination of RBD and nonparametric models, as for instance  in \cite{ICRA2010NguyenTuong_62320,wu2012semi}. 

In order to endow robotic systems with the ability to adapt to changing conditions, algorithms should be able to process data online, while taking advantage of knowledge already acquired, in the spirit of so-called transfer learning \cite{pan2010survey, bocsi2013alignment}.  Feasibility of online learning strongly hinges on the possibility to keep the computational complexity and memory storage requirements bounded as the number of data grows with time. To this propose several approaches have been proposed in the literature  \cite{Lawrence:2002,NIPS2000_1880}, \cite{Snelson06sparsegaussian}, \cite{Tresp:2000,Williams01usingthe,AE-MY-JH:11,Csato:2002} including algorithms explicitly developed to process data online, such as   \cite{nguyen2011incremental}, which selects  a sparse subset of training data points, and the local Gaussian process regression approach proposed in \cite{nguyen2009model}. 

In  \cite{gijsberts2011incremental} the complexity  is kept constant approximating the kernel function using the so-called ``random features'', \cite{rahimi2007random}, an approximation which will be exploited in this paper following also \cite{nguyen2011incremental, SEMIPARAMTERIC_2016,romeres2016onlineIcub}. It is worth noting that other approximation techniques are available in the literature, for instance: subset of data approximation, subset of regressors approximation,
conditional approximations, Nystrom approximation and relevance vector machine approach, see \cite{quinonero2005unifying}. Approximation of the kernel is equivalent to choosing a finite (and fixed) basis expansion for the unknown model; this expansion can be exploited to tackle estimation using Kalman filtering techniques \cite{JH-SS:10}, \cite{NicFer:1998:IFA_1779} and 
 \cite{Huber201485},  \cite{JP-MS:14}.  Online nonparametric learning using Gaussian processes calls also for online tuning of the kernel function, see for instance 
 \cite{romeres2016online,prando2016online} where marginal likelihood optimisation is used in the context of online system identification.

The main contributions of this paper are as follows:
\begin{itemize}
\item  Various models (parametric, nonparametric and semiparametric)  proposed in the literature \cite{ICRA2010NguyenTuong_62320,wu2012semi,SEMIPARAMTERIC_2016} are placed in a common framework through the so-called semiparametric models. This common framework is exploited to compare theoretically the semiparametric model wtih RBD mean (SP) and the semiparametric with RBD kernel (SPK)  and to show (see Section   \ref{sec:semiparametric} for details)  that the latter is to be preferred, complementing the findings from the experimental evaluation. Online learning is performed, following \cite{SEMIPARAMTERIC_2016}, exploiting the random features approach.
\item A new ``derivative-free'' modelling framework, which avoids the use of numerical derivatives, is proposed. In robot modelling, this framework contrasts the errors introduced into the system by computing numerically the velocities and accelerations starting from the measured (noisy) positions.
\item A thorough experimental study, based on real data from the iCub robot, is undertaken; classical methods as well as the newly proposed derivative-free methods are compared and analysed, both in terms of adaptation capabilities (how fast the algorithms can learn the dynamical models after a change in the experiental conditions) as well as in terms of steady state error. The experimental results show that the derivative-free methods proposed in this paper outperform classical schemes based on numerical derivatives. 
In doing so Cross Validation and Marginal Likelihood optimisation methods are compared for estimating the kernel hyperparametrs, suggesting that the latter should be preferred to the former.
\end{itemize}

The paper is organized as follows. In Section \ref{sec:prob_statement} the problem of inverse dynamics modeling is formalized. In Section \ref{sec:semiparametric} parametric, nonparametric and semiparametric models are introduced, while Section \ref{sec:model_approx} deals with kernel approximation which allows for the use of online algorithms. Section \ref{sec:derivative-free} introduces learning methods to avoid the use of numerical derivatives.  In Section \ref{sec: Simulations} the different online algorithms are tested in the inverse dynamics estimation of the robotic platform iCub. Finally, in Section \ref{sec:conclusions} conclusions and future works are drawn.

\section{Problem Statement}
\label{sec:prob_statement}

Assume we are given a robotic system with $\ndof$ degrees of freedom (DoF), and denote with $q(t) \in  \mathbb{R}^\ndof$, $t \in \mathbb{R}$,  the free coordinates describing the robot configuration. 
Starting from the laws of physics, which account for gravity, apparent forces, frictions and so on,  it would in principle be possible to write a (direct) dynamical model which, having as inputs the torques $y(t)$ acting on the robot joints,  outputs the  trajectory of the free coordinates (joint positions) $q(t)$. This is the so called ``direct dynamics''. 

However, for the purpose of control design, it is of interest to know which  torques $y(t)$  should be applied  in order to obtain a certain trajectory $q(t)$. 
 This can be achieved for example using inverse dynamics models, which can be  exploited, e.g., to determine the feedforward joint torques $y^d(t)$ which should be applied to follow  a desired trajectory $q^d(t)$; see Figure~\ref{fig_aplication}. \begin{figure}[hbtp]
\centering
\includegraphics[width=0.9\columnwidth]{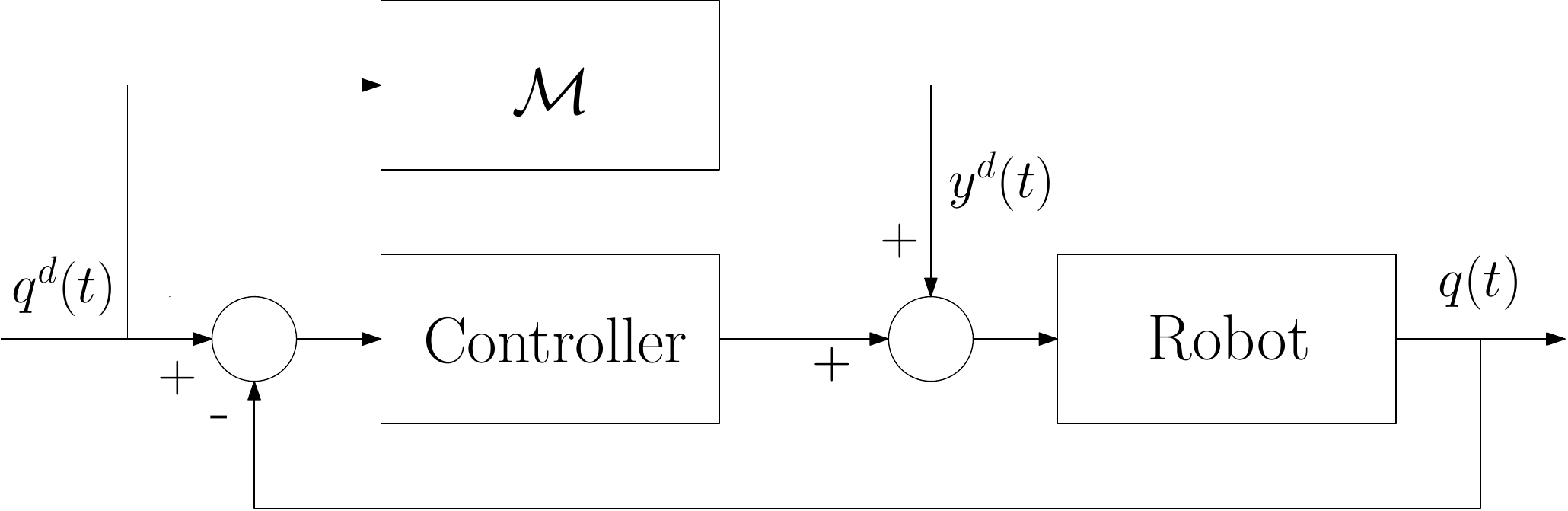}
\caption{Schematic for robot motion control.}\label{fig_aplication}
\end{figure} Clearly,  closing  a high performance loop quests for an accurate  inverse dynamics model  $\mathcal{M}$. While control is an important application of inverse dynamics models, it is certainly not their unique use. For instance, inverse dynamics models find important applications in modelling, detecting and estimating contact forces, see e.g.  \cite{IvaldiICRA2015}.

In practice an {\em inverse dynamics model}  should be learned   from a set of measured data $(y(t),q(t))$, with $t=1\ldots N$; this should be possibly performed online, updating the model as new data become available. An overview of model structures used for  inverse dynamics modeling will be provided in Section \ref{sec:semiparametric}  and the approximations involved in  online learning discussed in Section \ref{sec:model_approx}.

%

\begin{remark}[ Causality of Inverse Dynamics Models]
Strictly speaking, such models are not proper: the present torques depend on the joint positions, velocities and accelerations which corresponds to having knowledge about future temporal instants.  In this paper we shall either assume velocities and accelerations are given, or causal approximations will be considered, where the present torques only depend on present  and past joint trajectories. If such models have to be used to build feedforward control laws, it is well possible that the future (desired) trajectories are known and thus one can build, starting from data, non-causal models. On the other hand, when online identification is to be performed (which means that the underlying system is time-varying) there is the delay in updating the time-varying model. However, time-variability should be slow and, as such, it should not be a significant drawback. 
\end{remark}

\begin{remark}[ Stability of Inverse Dynamics Models]
A model of a  mechanical system   as a map from joint   torques  $y(t)$ (the inputs) to joint angular velocities $\dot{q}(t)$ (the outputs), is  passive. In the linear case this implies that the transfer from input to outputs is positive real and thus it is minimum phase. Nonlinear extensions of this concept are indeed possible, i.e. a passive nonlinear system has stable zero dynamics, see e.g. \cite{ByrnesIW1991}. Therefore, when considering inverse dynamics for (passive) mechanical systems, stability is not an issue. 
Stability of the inverse model may become critical when non minimum-phase systems are considered (e.g.  non-collocated systems); it would thus be  possible to consider acausal and stable functions of the reference trajectory to build a (stable) inverse dynamics model. This would be compatible with   online implementation provided the reference trajectory is known in a finite and sufficiently long future horizon. We shall not consider the latter case in this paper.
\end{remark}

The typical assumption that joint positions $q(t)$, joint velocities $\dot{q}(t)$, and joint accelerations $\ddot{q}(t)$ are measured considerably simplifies the identification procedure; the latter are stacked in the so-called 
``input location'' vector 
\al{
\label{eq:x}
x(t):=[\,q^\top(t)\, \dot{q}^\top(t)\, \ddot{q}^\top(t)\,]^\top\in\mathbb{R}^m
}
 with $m=3n$. Under this assumption, the inverse dynamics of a rigid body  can be written in a linear form once a suitable overparametrization is introduced, see Section \ref{subsec:par}. Unfortunately, measuring velocities and accelerations is often unrealistic and one has to resort to numerical differentiation schemes, which may be prone to large errors  in the presence of measurement noise. 
 In Section \ref{sec:derivative-free}, we shall analyze an alternative model which  only exploits the past history of the joint positions. In this way, $\mathcal{M}$ is a model with the past history of the joint positions as input and with the applied torques as output.  As we shall see this choice compares favourably with standard approaches in the literature.

\section{Model Classes}
\label{sec:semiparametric}

In this Section we briefly review the typical model classes used to learn the inverse dynamics of a robot, that is the linear parametric model, whose structure is given by the physics, the nonparametric model, whose structure is learnt from the data, and the semiparametric models which are a combination of the first two models.  Throughout this Section, Gaussian processes are indexed in $\mathbb{R}^m$, take values in $\mathbb{R}^n$ and are zero mean. The symbol 
$e(t)$ denotes a zero mean white Gaussian noise with covariance matrix $\sigma^2 I_n$. Most of the  models introduced in this Section are equipped with a so-called {\em hyperparameters vector}. We address the problem of estimating this vector in Section \ref{sec:model_approx}.

\subsection{Linear parametric model} \label{subsec:par}
The rigid body dynamics (RBD) of a robot is described by the equation
\begin{equation}
\label{phys_model} 
y(t)= M(q(t)) \ddot{q}(t)+C(q(t),\dot{q}(t)) \dot{q}(t) +G(q(t))
\end{equation}
where $M(q(t))$ is the inertia matrix of the robot, $C(q(t),\dot{q}(t))$ the Coriolis and centripetal forces and $G(q(t))$ the gravity forces, \cite{siciliano2010robotics,taylor2005classical}. The terms on the right hand side of  (\ref{phys_model}) can be rewritten as $\psi^\top(x(t))\pi$ which is linear in the robot (base) inertial parameters vector $\pi\in\R^{p}$. The map $\psi \,: \,\R^m \rightarrow \R^{p\times \ndof}$ is the known RBD regressor which is a nonlinear function of the input locations vector described by $x(t)$. 
Therefore, the RBD model is equivalent to 
\begin{equation}
\label{par_regr_model}
y(t)= \psi^\top(x(t)) \pi+e(t)
\end{equation}
where $e(t)$ includes the nonlinearities of the robot that are not modeled by the rigid body dynamics (e.g. actuator dynamics, frictions, etc.). 

Since the RBD model is physics-based, it should, in principle, describe the robot dynamics for all the desired trajectories, leading to good generalization and global approximation properties. A known issue of this model (see e.g., \cite{hollerbach2008model}) is that the problem of determining $\pi$ from measured data $y(t)$ is usually ill posed and the matrix $\psi(x(t))$ is rank deficient. Possible solutions in system identification are either the design of efficient experiments to collect data sufficiently rich to excite the highest number of modes of the system or dedicated experiments which are good to estimate parameters separately. 

Another drawback of such model class is that it is prone to undermodeling (as it may not capture non-linear friction effects, effects of non-rigidities etc.) that may ultimately result in severely biased estimated models. 

This model will be denoted as ``P'' (Parametric).
 
\subsection{Nonparametric Model}
Following the ``Gaussian Process'' framework \cite{Rasmussen}, the robot inverse dynamics can be modeled postulating that:
\begin{equation}
\label{nonpar_model} 
y(t)=f(x(t))+e(t)
\end{equation}
where $f(\cdot)$ is a Gaussian process with covariance function (i.e. kernel function)\footnote{To be precise, the kernel function and the covariance function coincide up to the scaling factor. However, to ease the exposition we will follow this abuse of terminology.} $\mathrm{cov}[f(x(t)), f(x(s))]=K(x(t),x(s))$:

\begin{equation}
\label{eq:kernel_nonparametrico_icub}
K(x(t),x(s))=\rho^2  K_G(x(t),x(s))I_\ndof.
\end{equation}
 Several choices are possible for the kernel $K$ (see e.g. \cite{Rasmussen,GP-MHQ-AC:11}), which can be exploited to encode specific model structures. However, in this paper we shall only consider   $K$ to be the Gaussian kernel\footnote{ In order to avoid scaling issues, each component of the input location vector $x(t)$ is normalised to have unit standard deviation. Some preliminary tests showed worse performance when using an independent  width for each component of $x(t)$, which is probably to be attributed to overfitting due to the large number of hyperparameters.}
\begin{equation}
\label{gaussian_kernel}
K_G(x(t),x(s))=e^{-\frac{\|x(t)-x(s)\|^2}{2\tau}}
\end{equation}
where the hyperparameter $\tau$ is the kernel width. The latter is a typical choice to correlate the input locations for learning the inverse dynamics, \cite{gijsberts2011incremental,ICRA2010NguyenTuong_62320,wu2012semi}. 
$\rho^2$ plays the role of scaling factor. The vector $\eta:=[\,\rho^2\, \tau\, \sigma^2\,]$ is referred to as {\em hyperparameters vector}. This model will be denoted as ``NP'' (NonParametric).


This class of models is known to have high flexibility and prediction performance (see e.g. \cite{Rasmussen}) since the dynamics are extrapolated directly from the experimental data, without making any unrealistic approximation on the physical system (e.g. assuming linear frictions models, ignoring the dynamics of the hydraulic actuators, etc.). Nevertheless,
nonparametric models deteriorate their performance when predicting unseen data which are far ( in the Euclidean metric) from
those visited in the training dataset.

\subsection{Semiparametric model with RBD mean} 
\label{sec:SP}
The semiparametric model is an attempt to take advantage of the global property of the parametric model as well as of  the flexibility of the nonparametric model; the first possibility is to embed the former as a mean term into the latter. Thus, the inverse dynamics is modeled as 

\begin{equation} 
\label{eq:SP-Mean_model_statistics} 
\begin{array}{rl} 
y(t)=& \psi^\top (x(t))\pi+f(x(t))+e(t)
\end{array}
\end{equation}
 where $\pi$ is the vector of inertial parameters, $\psi(x(t))$ is the RBD regressor and  $f(\cdot)$ a Gaussian process with kernel function as in  \eqref{eq:kernel_nonparametrico_icub}. 
 
 At this point two  alternatives are possible. The first and most principled one is to treat $\pi$ as an unknown hyperparameter. Model \eqref{eq:SP-Mean_model_statistics} with this hypothesis will be denoted as ``SP'' (SemiParametric). In this case $\eta:=[\, \pi \, \rho^2 \,\tau\, \sigma^2\,]$. 
 
A suboptimal alternative but often applied in the literature is to assume that $\pi$ is known, here denoted by $\hat\pi$. This could be  possibly estimated using some preliminary experiment as in  \cite{ICRA2010NguyenTuong_62320} or estimated via Least Squares as in \cite{SEMIPARAMTERIC_2016} or it can be given from some expert knowledge.  Model \eqref{eq:SP-Mean_model_statistics} with this hypothesis will be denoted as ``SP2''. In this case $\eta:=[\,  \rho^2 \,\tau\, \sigma^2\,]$.

\subsection{Semiparametric model with RBD kernel} 
An alternative possibility for combining the parametric and nonparametric models is to incorporate the RBD structure in the kernel, \cite{ICRA2010NguyenTuong_62320}. 
The debate as to whether one should account for prior knowledge  as a mean term or as a structured kernel is often encountered in Gaussian Process regression. We shall come back to the relation between these two alternatives in Proposition \ref{proposition:connection_SPK_SPmean} and the discussion which follows that proposition.

Therefore, the inverse dynamics is modeled as
\al{y(t)=f(x(t))+e(t) \label{modelloSPK}}
where $f(\cdot)$ is a Gaussian process with kernel function 

\al{\label{eq:SP-Kernel_model_statistics}  K(x(t), x(s)) & = \gamma^2 \psi(x(t))^\top \psi(x(s)) +\rho^2 K_G(x(t),x(s)) I_n.} This model will be denoted as ``SPK'' (SemiParametric Kernel). The hyperparameters vector is $\eta:=[\,\gamma^2\, \rho^2\, \tau\, \sigma^2\,]$.

It is worth noting that in the case $\rho=0$ we obtain the parametric model (\ref{par_regr_model}) where the inertia parameters vector is now modeled as a Gaussian random vector with zero mean and covariance matrix $\gamma^2 I_p$.

\begin{remark}
 Note that  equation  \eqref{eq:SP-Kernel_model_statistics} is derived assuming  $\pi$ is  a zero mean Gaussian vector with variance $\gamma^2 I $. 
  Of course if prior knowledge was available on the expected scale of the components of $\pi$, this could be included in the prior variance introducing   a diagonal scaling matrix $D$, so that $\pi \sim {\cal N}(0, \gamma^2D)$. 
 
We have also tested an alternative version in which   the diagonal matrix $\gamma^2D$ has been estimated via Marginal Likelihood together with all other hyperparameters. However the overall performance is worse than when assuming $\pi \sim {\cal N}(0,\gamma I)$ and estimating only $\gamma$, which is probably due to the fact that 
 this extra freedom results in a less favourable bias-variance tradeoff.
\end{remark}

The semiparametric model with RBD kernel  (SPK) is  connected with the RBD mean (SP) model introduced in Section \ref{sec:SP}. To make this connection sharp we shall refer to both the minimum variance estimators (see Section \ref{sec:model_approx}) obtained with these models, as well as to the log-likelihood functions which can be used to estimate the hyperparameters (see Section \ref{subsec:hyperparameter_estimation}). In particular, let us stack the available data $y(t)$, $t=1\ldots N$ in the vector $\mathbf{y}$ and stack correspondingly  the regressors $\psi^\top(x(t)) $ and $f(x(t))$ in the matrix $\Psi$ and vector $\mathbf{f}$, respectively.  Moreover, we define $\mathbf{K}(\mathbf x,\mathbf x)=\mathrm{cov}[\mathbf f,\mathbf f]$, $L_{SPK}(\mathbf y)$ and $L_{SP}(\mathbf y)$ as the negative marginal  log-likelihoods $-2\log\, p_\eta({\mathbf y})$ of data ${\mathbf y}$ as a function of hyperparameters $\eta$ for models SPK and SP, respectively, and \begin{equation}
\label{eq:profile_marginal_lik}
\hat L_{SP}(\mathbf{y}) : = \mathop{\rm min}_{\pi} L_{SP}(\mathbf{y}) = L_{SP}(\mathbf{y})_{|  \pi= \hat\pi_{WLS}}
\end{equation}
as the \emph{profile} marginal log-likelihood, where $\hat \pi_{WLS}$ denotes a suitable weighted least squares estimate of $\pi$ (see Appendix \ref{AppA} for the precise definition of $\hat \pi_{WLS}$).
The following proposition establishes the link between the semiparametric with RBD  mean (SP) and semiparametric with RBD kernel (SPK) estimators as the prior on $\pi$ becomes uninformative (i.e. as $\gamma^2 \rightarrow \infty$),  showing that also in this asymptotic regime the two models induce a different marginal likelihood  (see eq. \eqref{Diff-SP-SPK}). The take-home message of this proposition is that the SPK model better handles uncertainties and should thus be preferred; a more in depth  discussion  will  be provided in Section \ref{app:discussion}. The reader is also referred to \cite{RUSSU2012189} where a similar discussion has been provided in a completely different context.

\begin{proposition}
\label{proposition:connection_SPK_SPmean}
Consider the semiparametric models SP in \eqref{eq:SP-Mean_model_statistics} and SPK in \eqref{modelloSPK}, their respective negative marginal  log-likelihoods $L_{SP}(\mathbf y)$ and $L_{SPK}(\mathbf y)$ and the profile marginal likelihood $\hat L_{SP}(\mathbf{y}) $ defined in \eqref{eq:profile_marginal_lik}. Assume that $\rho^2,\tau,\sigma^2$ are fixed. Then, the following two statements hold:
\begin{enumerate}
\item[(i)] the minimum variance estimator  of model SP  and model SPK coincide  when $\gamma^2 \rightarrow \infty$;

\item[(ii)] in the limiting case of $\gamma^2 \rightarrow \infty$ the two log-likelihoods $L_{SPK}(\mathbf y)$ and $\hat L_{SP}(\mathbf{y})$ differ in a nontrivial term ${\rm log}({\rm det}( \Psi^\top R^{-1} \Psi))$ which leads to different location in their minima. That is:
\begin{equation}\label{Diff-SP-SPK}
\begin{array}{c}
  \displaystyle\mathop{\rm lim}_{\gamma^2 \rightarrow\infty} L_{SPK}(\mathbf y) - \hat L_{SP}( \mathbf y)  - p\, \log \,\gamma^2   \quad \quad\quad  \\
=  {\rm log}({\rm det}( \Psi^\top R^{-1} \Psi))
 \end{array}
\end{equation}
where $R=\mathbf{K}(\mathbf x,\mathbf x) + \sigma^2 I$.
\end{enumerate}
\end{proposition}
 \begin{proof} See Appendix \ref{app:proof}.
 \end{proof}
   \subsection{Discussion of Proposition \ref{proposition:connection_SPK_SPmean}}
\label{app:discussion}
 Proposition \ref{proposition:connection_SPK_SPmean} shows that, when estimating hyperparameters using the marginal likelihoods (as discussed in \ref{subsec:hyperparameter_estimation}), the two models (SP and SPK) are not equivalent even under the assumption that  a non informative prior on $\pi$ is used, i.e. $\gamma^2 \rightarrow \infty$. In fact, the two marginal likelihood functions differ by a nontrivial term ${\rm log}({\rm det}( \Psi^\top R^{-1} \Psi)) $ which may change the location of the minima. In particular the latter term accounts for the uncertainty in estimating the term $\pi$. 
 
When the SP model is used, the hyperparameters $\sigma^2, \rho^2, \tau$ are estimated minimising 
$$
\hat L_{SP}(\mathbf{y}) = {\rm log}({\rm det}(2\pi R)) + (\mathbf{y}-\Psi \hat \pi_{WLS} )^\top R^{-1} (\mathbf{y}-\Psi \hat \pi_{WLS} ).	
$$ Doing so, $\sigma^2, \rho^2, \tau$ are chosen so as to fit  with $R$ the ``sample'' covariance 
$(\mathbf{y}-\Psi \hat \pi_{WLS} )(\mathbf{y}-\Psi \hat \pi_{WLS} )^\top$. The drawback of this choice is that the component of the variance along
$\Psi$ may be underestimated (i.e. $R$ too small along $\Psi$). On the contrary, when optimising $L_{SPK}({\mathbf y})$ 
the term
 $$
 {\rm log}({\rm det}( \Psi^\top R^{-1} \Psi)) = - {\rm log}({\rm det}( Var\{\hat\pi_{WLS}\})
$$ favour values of $\sigma^2, \rho^2, \tau$ which make ${\rm log}({\rm det}( Var\{\hat\pi_{WLS}\})$ large or, equivalently, $R^{-1}$ small (i.e. $R$ large) along the direction of $\Psi$.

Clearly this goes in the opposite direction and avoids the risk alluded at above (i.e. $R$ too small along $\Psi$).   
Thus, to summarise, the term ${\rm log}({\rm det}( \Psi^\top R^{-1} \Psi))$ serves as a correction which accounts for the uncertainty of $\hat\pi_{WLS}$. Thus it is fair to say that model SPK deals more ``cautiously'' with uncertainty than model SP. This is reflected in the simulation results (see e.g. Figure \ref{fig:prediction_error_boxplot}) where SPK slightly outperforms SP (both endowed with ML hyperparameters estimation, i.e. SP-ML and SPK-ML).

\section{ Kernel Approximation and Online Learning}
\label{sec:model_approx}

In this Section we review the online learning problem using the model classes of Section \ref{sec:semiparametric}. Since it is well known how to perform online learning using the parametric model (\ref{par_regr_model}), we will focus on the other models. On the other hand, model (\ref{par_regr_model}) can be understood as a special case, for instance, of model (\ref{eq:SP-Mean_model_statistics}). Online learning using nonparametric and semiparametric models can be performed  by applying the Gaussian regression framework, \cite{Rasmussen}. In order to do so, two issues have to be faced  which will be addressed in Section \ref{sec:model_approx/kernel_approx} and Section \ref{subsec:hyperparameter_estimation}, respectively. First, following \cite{SEMIPARAMTERIC_2016,romeres2016onlineIcub}, an approximation of the kernel is introduced to allow online learning with constant complexity. Second, two commonly used approaches are very briefly reviewed  to estimate the hyperparametes vector $\eta$, namely Cross Validation and Marginal Likelihood optimisation. 

\subsection{ Kernel approximation}
\label{sec:model_approx/kernel_approx}

A typical route followed in   machine learning is to approximate the  kernel with a low rank factorisation of the form

\begin{equation}
\label{approx_K}
K_G(x(t),x(s)) \approx \phi^\top(x(t))  \phi(x(s))
\end{equation}
where $\phi\,: \, \mathbb{R}^m \rightarrow \mathbb{R}^{2d}$ is a suitable vector of ``basis functions''  that have to be properly defined. Ideally,  $\phi(x(t)) $ should be built from the eigenfunctions of the kernel matrix; however this optimal approximation is computationally  unfeasible in a recursive algorithm\footnote{An exception is the case when the eigenfunctions of the kernel are available analytically.}. Therefore, we rely on the random features approximation proposed in~\cite{rahimi2007random} and exploited for online inverse dynamics modelling in \cite{gijsberts2011incremental,SEMIPARAMTERIC_2016,romeres2016onlineIcub}.

According to the random feature approximation, the basis functions vector for the Gaussian kernel is defined as\footnote{This construction is based on the fact that any bounded and positive semidefinite kernel is the Fourier transform of a non-negative and integrable function (Bochner's theorem), which thus induces a probability distribution (the function $p(\omega)$ below equation \eqref{eq:basis_function}). See \cite{rahimi2007random} for details. Another interpretation is that $p(\omega)$ represents (up to a constant factor) the power spectral density \cite{papoulis2002probability} of the stationary kernel $K_G$.}

\al{
\label{eq:basis_function}
\phi(x)=  \frac{1}{\sqrt{d}}
& \begin{array}{ccc} 
\left[    \cos  \left(\frac{\omega_1^\top x}{\tau} \right) \right. & \ldots &  \cos  \left(\frac{\omega_d^\top x}{\tau} \right)   \\
    \end{array} \nn \\
& \hspace{0.5cm}\begin{array}{ccc} 
 \sin \left(\frac{\omega_1^\top x}{\tau} \right)  & \ldots  & \left. \sin  \left(\frac{\omega_d^\top x}{\tau} \right)   \right]^\top
\end{array}}
with $w_k \sim  p(\omega) = \exp(-\|\omega\|^2 /2) /(\sqrt{2\uppi} )^m$, $k=1,\ldots,d$. 
 Notice that if $d\rightarrow \infty$ then the approximation is almost surely exact.
As a consequence, the parameter $d$ has to be chosen to trade-off accuracy of the approximation and  computational complexity. Finally, it is worth noting that $\phi(x)$ depends on the width of the Gaussian kernel $\tau$.

Using the approximation in (\ref{approx_K}) for model (\ref{modelloSPK}) is equivalent to consider the approximation
\al{f(x(t))\approx \psi^\top (x(t))\pi+\left[\phi^\top (x(t))\otimes I_n\right] \pi_{NP}} where $\pi$ and $\pi_{NP}$ are independent Gaussian random vectors with zero mean and covariance matrices $\gamma^2 I_p$ and $\rho^2 I_{2d}$, respectively. Finally, by defining \al{\varphi^\top(x(t))&=\left[\begin{array}{cc}\psi^\top (x(t)) & \phi^\top (x(t))\otimes I_n  \end{array}\right]\nn\\ 
\theta &=\left[\begin{array}{cc}\pi^\top &\pi_{NP}^\top \end{array}\right]^\top
\nn} the approximation of model (\ref{modelloSPK}) is
\al{\label{modello_approx}y(t)=\varphi^\top(x(t))\theta+e(t).}  It is possible to show that all the models in Section \ref{sec:semiparametric} can be approximated with a model in the form (\ref{modello_approx}). A derivation of all these approximations can be found in 
 \cite{romeres2016onlineIcub}.

\subsection{Online learning}
\label{sec:model_approx/online_learning}
 
 Next, we address the problem of  online learning. For simplicity of exposition, we consider the NP model in (\ref{nonpar_model}). It is well known that, given  data $(y(t),x(t))$, $t=1\ldots N$, the minimum variance estimator of $f$ can also be expressed as  the solution of the Tikhonov regularization problem,  
\cite{Rasmussen}, 
\begin{equation}
\label{reg_pb_inf}
\hat f =\underset{f\in\mathcal{H} }{\mathrm{argmin}} \frac{1}{\sigma^2}\sum_{t=1}^N \|y(t)-f(x(t)) \|^2+\frac{1}{\rho^2}\| f\|^2_{\mathcal{H}}
\end{equation}

where $f$ belongs to the reproducing kernel Hilbert space $\mathcal{H}$, \cite{wahba1990spline}, with reproducing kernel function $K_G$ and norm $\|\cdot\|_{\mathcal{H}}$. 

It is not difficult to see that the Tikhonov regularization problem (\ref{reg_pb_inf}) takes a similar form for all the models of Section \ref{sec:semiparametric}.  The solution of (\ref{reg_pb_inf}) takes the general form
\al{\hat f(x)=\sum_{t=1}^N \alpha_t K(x,x(t))}
where $\alpha_t$'s are uniquely determined from $(x(1),y(1)), \ldots ,(x(N),y(N))$. Accordingly, the solution of (\ref{reg_pb_inf}) is unique regardless of the excitation properties of $x(t)$. Moreover, it is possible to prove that $\alpha_t$'s take finite value, so that $\hat f$ maps bounded trajectories into bounded torques/forces. It is worth noting that the number of coefficients $\alpha_t$ coincides with the number of the data points/samples, i.e. $N$, making the estimator intractable for an online (recursive) solution. On the other hand, using the approximated model (\ref{modello_approx}), the minimum variance estimator of $\theta$ is the solution to the following Regularized Least Squares problem 

\al{\label{ReLS} \hat \theta=\underset{\theta}{\mathrm{argmin}} \frac{1}{\sigma^2}\sum_{t=1}^N \|y(t)-\varphi^\top(x(t))\theta\|^2+\theta^\top W\theta}

where  $W$ is the inverse kernel matrix induced by the approximation. Its optimal solution can be computed recursively, whenever new data becomes available, through the well known Recursive Least Squares algorithm, see e.g   \cite[Chapter 11]{Ljung:99} and \cite{RPEM}. In practice, the implementation of this algorithm uses Cholesky-based updates \cite{bjorck96},  which have robust numerical properties. The computational complexity of each update is $O(\bar p^2)$ where $\bar p $ is the dimension of vector $\theta$;  the cost of evaluating the model output is 
$O(\bar p)$ for each output and thus $O(\bar p n)$  if $n$ wrenches (torques/forces) are to be computed. This computational complexity is compatible with online implementation for state-of-the-art computation facilities.

\subsection{Hyperparameter vector estimation}
\label{subsec:hyperparameter_estimation}
Nonparametric and semiparametric models are characterized by the hyperparameters vector $\eta$, which is not known  and needs to be estimated from the data. Typically, two approaches are considered to address this problem.

A first possible method is called \textit{Cross Validation} (CV). The goal is to obtain an estimate of the prediction capability of the model on future data for different choices of the hyperparameters vector $\eta$. Hold-out cross validation approach is a possible version of CV, see e.g. \cite[Chapter 6]{james2013introduction}. The dataset is split in two, the training set and the validation set. The optimal $\eta$ is  given by optimizing an estimate of the mean squared error in the validation set
\al{\label{CVcost}\hat \eta_{CV}=\underset{\eta\in\Omega}{\mathrm{argmin}}\; \widehat{\mathrm{MSE}}(\eta).}
In practice this approach is limited to the estimation of a small number of hyperparameters since the minimization in \eqref{CVcost} is typically performed by gridding the search space $\Omega$.

A second method is offered directly by the Gaussian regression framework. The {\em marginal likelihood} (ML), denoted by $p_\eta(\mathbf{y})$, expresses the likelihood of the hyperparameters vector $\eta$ on the data $\mathbf{y}:=[\, y(1)^\top\, \ldots \, y(N)^\top]^\top$, once the parameter $\theta$ has been integrated out. Under model (\ref{modello_approx}), the latter can be computed in closed form, as
discussed in \cite{Rasmussen}. To conclude, the optimal $\eta$ is then given by solving 

\al{\hat \eta_{ML}=\underset{\eta\in\Omega}{\mathrm{argmax}}\;p_\eta(\mathbf{y}).} 

In Section \ref{sec: Simulations} we shall thoroughly compare on real data these two approaches for hyperparameter estimation. The experimental results, in this case, show that Marginal Likelihood optimisation outperforms Cross Validation. 

\section{Derivative-free Learning} \label{sec:derivative-free}

In Section \ref{subsec:par} we have seen that the rigid body dynamics suggests that the inverse dynamics is described by a map from the input locations vector $x(t)=[\,q^\top(t) \,\dot{q}^\top(t) \, \ddot{q}^\top(t) \, ]^\top$ 
composed by joint positions-velocities-accelerations to the joint torques applied to the $n$ joints. This hypothesis has been exploited in all the model classes in Section \ref{sec:semiparametric}. However, it is often the case that joint velocities and accelerations cannot be measured from the robot; rather they  are estimated using numerical differentiation schemes from the (noisy) measurements of the joint positions. As a consequence, failure to properly handle noise in the measurement  may severely hamper the final solution.   This is a very well known and highly discussed problem, see e.g., \cite{siciliano2010robotics,hollerbach2008model,kozlowski2012modelling,craig2005introduction,Nguyen-Tuong2011} and it is usually  partially addressed by ad-hoc filter design. However, this requires users' knowledge and experience in tuning the filters' parameters.

In this Section we propose a new methodology, which avoids the use of numerical pre-differentiation issues, by learning the inverse dynamics without using the estimated velocities and accelerations.  We shall thus assume that only  joint torques $y(t)$ and  joint positions $q(t)$, $t = 1\ldots N$, are measured. Let 
\al{q(t^-) := [\,q_1^\top(t^-) \, q_2^\top(t^-) \, \ldots \, q_n^\top(t^-) \,]^\top \in \R^{(M+1) \ndof}, }
and 
\al{q_i(t^-) := [\,q_i^\top(t) \, q_i^\top(t-1) \, \ldots \, q_i^\top(t-M) \,]^\top \in \R^{M+1 }, }
be the vector of the past joint positions and the vector of the past of the $i$-th joint position, respectively, 
in the time window $[t-M,t]$ where $M$, sufficiently large, has been fixed. Our aim is to model the inverse dynamics as a map from    the past joint positions $q(t^-)$ to the joint torques. In particular, we shall postulate that output $y(t)$ can be written as a non-linear function of a ``features vector'' ${\xi}(t):=[\,\xi_1^\top(t)  \, \xi_2^\top(t)  \, \ldots \,\xi_n^\top(t) \,]^\top$, defined as a linear function 
of past measurements $q_i(t^-)$ 
\begin{equation}
\label{eq:input_locations_derivativefree_general}
{\xi}_i(t) = R \, q_i(t^-)
\end{equation}
where $R \in \R^{k \times (M+1)}$. Using the features \eqref{eq:input_locations_derivativefree_general} can be seen as a reduced rank regression problem. We shall discuss later on the choice of the number $k$  (rows of $R$), i.e. the number of features which are used by model \eqref{eq:derivative_free_model} below\footnote{For example, in the standard case, where we have the input locations vector, the features are   $k=3$: joint positions, velocities and accelerations.}.  The choice of $k$ could also be performed resorting to Bayesian arguments, as for instance done in \cite{BSL_JOURNAL} in a similar low-rank regression problem in the context of  sparse-low rank dynamic network models. 

In particular, we shall  model the inverse dynamics with the NP model 
\begin{align}
\label{eq:derivative_free_model}
y(t) = f(\xi(t)) + e(t)
\end{align}
where  $f(\cdot)$ is a Gaussian process with kernel function 
\al{ K(\xi(s),\xi(t))=\rho^2K_G(\xi(s),\xi(t)) I_n .} 

Then, the online learning algorithm is similar to the one sketched in Section \ref{sec:model_approx}. The only difference is that the ``standard'' input locations vector $x(t)$  in \eqref{eq:x} is replaced with $\xi(t)$ in   \eqref{eq:input_locations_derivativefree_general}. Special cases which one may consider are $R=I$ so that  $\xi(t)$ in \eqref{eq:input_locations_derivativefree_general} coincides with the past measurements (i.e. no dimensionally reduction is performed); as an alternative (we shall come back to this later on) the rows of the matrix $R$ may compute, using causal filtering operations, an approximation of the derivatives  of $q(t)$ (i.e. $\dot{q}(t), \ddot{q}(t)$ and possibly higher order derivatives). We may thus say that the approach considered in this section generalises what seen in the previous ones.

In the following, the matrix $R$  in (\ref{eq:input_locations_derivativefree_general}) will be described  by a set of  hyperparameters that have to be estimated from the data; for instance, for the NP model we shall have $\eta:=[\,\rho^2\, \tau\, \sigma^2\, \delta^\top \,]$, where $\delta$ is the  vector containing the hyperparameters of $R$. Models P, SP, SP2 and SPK can be derived in a similar way.

We shall now consider three possible choices (but of course many others would be possible) for the structure of the matrix $R$.

\subsection{Derivative-free features}
The simplest choice is to take $R=I_{M+1}$ that is the features vector coincides with the $M$ past measured joint positions. Thus $k=M+1$. These features will be denoted as ``Derivative-Free" (DF).
As an alternative it is possible to choose 
\al{	\label{eqRdiag}R=\left[\begin{array}{cccc} r_1 & 0&  \ldots & 0\\ 0 & \ddots & & \vdots  \\  \vdots &  & \ddots & 0\\
0 & \ldots & 0 & r_{M+1}\end{array}\right]}
with $r_i \in \R, i = 1,.., M+1$. That is, the features vector is a weighted version of the past measured joint positions. These features will be denoted as ``Derivative-Free Weighted'' (DFW).

\subsection{{Derivative-free features with reduced rank }}
\label{subsec:derivative_free/reduced_rank}

An alternative choice is to take the number of features $k$ smaller than $M+1$. For instance, the physics suggests that the right number of features should be equal to 3 (position, velocity, acceleration). Therefore, the role of the features is to compress the useful information available in $q(t^-)$ so as to render the learning procedure more robust. This means that
\begin{equation}
\label{eq:input_locations_features_free}
 R = \left[\begin{array}{ccc} r_{1}^\top\\  r_{2}^\top \\ \vdots \\  r_{k}^\top \end{array}\right],  \quad \|r_i\| = 1
\end{equation}

where $k\ll M$, $r_i \in \R^{M+1}$ and the vector of hyperparameters is $\delta:=[r_1,..,r_k]$, i.e. $R$ is fully parameterised. In practice it makes also sense to impose some constraints in the $r_i$'s. In particular $r_1^\top := [1\, 0 \ldots \, 0]$ so that the first feature is the measured position. It would also be reasonable to enforce other constraints, for instance imposing that $R$ is an orthogonal projection. Notwithstanding its importance, we shall not delve further into this issue. 

The number of features $k$ is chosen by the user and it will be subject of empirical analysis in Section \ref{sec: Simulations}. These features will be denoted as ``Derivative-Free Reduced rank'' (DFR).

Notice, that the features DFR include all the possible linear and causal numerical differentiation and filtering operations. The price of this generality is that a large number of hyperparameters has to be estimated, i.e. $(k-1)(M+1)$. As it is well known in optimization, this might lead to local minima problems. In order to overcome this issue one might resort to regularization techniques on the hyperparameters or to set appropriate initial conditions.

\subsection{Structured derivative-free input locations with reduced rank}
The last idea is to consider only $k=3$ features: the first will be the position $q(t)$ while the other two will attempt to estimate explicitly   velocities and accelerations. We know that the joint velocities and accelerations can be computed by a first order backward difference and by a second order backward difference, respectively, both  filtered by a first order low pass filter, that is
\begin{align*}
\dot{q}_i(t) &\approx  \frac{1-z^{-1}}{T_s} \frac{1}{1-\beta_1z^{-1}} q_i(t)\\
\ddot{q}_i(t) &\approx \frac{1-2 z^{-1}+z^{-2}}{T_s^2} \frac{1}{1-\beta_2z^{-1}} q_i(t)
\end{align*}
where $z^{-1}$ is the backward shift operator, $T_s > 0$ is the sampling time and $ 0 < \beta_1, \, \beta_2 < 1$ represent the poles of the filters. 

We resort to a partial fraction decomposition to rewrite the above expressions as a function of $q(s^{-})$, that is:

\begin{align}
\dot{q}_i(t) & \approx \alpha_1 q_i(t) + \sum_{t=1}^M \alpha_1 \beta_1^{t-1}( \beta_1 -1) q_i(s-t) \label{eq:velocity_fraction_decomposition_icub}\\
\ddot{q}_i(t) & \approx \alpha_2 q_i(t) + \alpha_2(\beta_2 -2)q_i(s-1) \notag \\
& \quad + \sum_{t=2}^M \alpha_2 \beta_2^{t-2}(\beta_2^2 -2\beta_2 + 1) q_i(s-t)
\end{align}
where $\alpha_1 = 1/ T_s$ and $\alpha_2 = 1/ T_s^2$. Here, we exploited the fact that a (stable) low-pass filter can be approximated by finite impulse response (FIR) filter with length $M$, where the latter is chosen sufficiently large. Accordingly, we have 
\begin{align}
 R = 
\begin{bmatrix}
1 & 0  & \ldots & \ldots & 0  \\
\alpha_1  & \alpha_1(\beta_1 -1)  & \ldots & \alpha_1 \beta_1^{t-1}(\beta_1 -1)  & \ldots \\
\alpha_2   & \alpha_2(\beta_2 -2)  & \ldots & \alpha_2\beta_2^{t-2}(\beta_2^2 -2\beta_2 + 1) & \ldots
\end{bmatrix}.
 \notag \\
 \label{eq:input_locations_features_vel_acc}
\end{align}

These features are denoted as ``Derivative-Free Structured Reduced rank'' (DFSR) since the features have a structure to resemble the joint positions, velocities and accelerations.

A nice property of this characterization is that the number of hyperparameters in $R$ is small, in this case $\delta = [\alpha_1,\alpha_2,\beta_1,\beta_2] \in \mathbb{R}^4$,   independently of the length of the past temporal lags $M$, which can be arbitrarily chosen.


\section{Inverse Dynamics Learning on iCub} \label{sec: Simulations}

In this section, we shall provide a thorough experimental evaluation of all the models discussed in Section \ref{sec:semiparametric}, also endowed with derivative-free input locations as discussed in Section \ref{sec:derivative-free}, for online learning of the inverse dynamics; all the algorithms have been tested using real data from the  iCub robot  \cite{metta2010icub}, which is  a full-body humanoid robot with 53 degrees of freedom. 

In this work we consider inverse dynamics modelling of the right arm. Therefore the free coordinates    $q(s) \in \mathbb{R}^4$ are the angular positions of the 3 shoulder joints and of the elbow joint, for a total of 4 degrees of freedom. When needed,  joint positions have been numerically differentiated to obtain joint velocities and accelerations by the Authors of \cite{SEMIPARAMTERIC_2016} using the standard  adaptWinPolyEstimator\footnote{Available at \\$wiki.icub.org/brain/adaptWinPolyEstimator\_8cpp\_source.html$} module of the open source iCub project. The differentiation procedure consists in applying a time-varying linear filter based on the work  \cite{janabi2000discrete}, that is actually implemented at a higher rate before downsampling the signals. 

 
The outputs  $y(s) \in \mathbb{R}^6$, are the 3 forces and 3 torques components measured by the six-axes force/torque (F/T) sensors embedded in the shoulder of the iCub's arm, see Figure \ref{fig_right_arm}.

\begin{figure}[hbtp]
\centering
\includegraphics[width=\columnwidth]{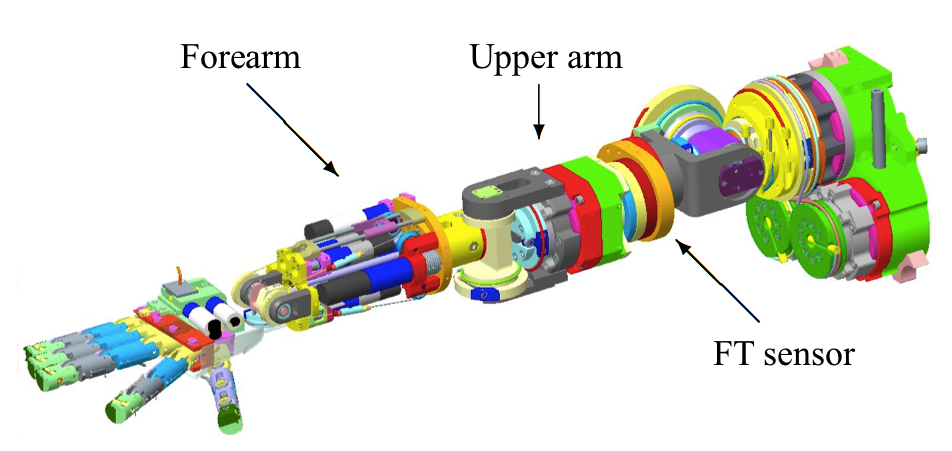}
\caption{iCub's right arm.}\label{fig_right_arm}
\end{figure}


Notice that the measured forces/torques are not the applied joint forces and torques and, as such, the model we learn is not, strictly speaking, the inverse dynamics model. Yet, as explained in \cite{ivaldi2011computing}, the feedforward joint torques can be determined from    components  (forces and  torques) of $y(s)$. Indeed, such model has been used in the literature as a benchmark for the inverse dynamics learning, \cite{gijsberts2011incremental}, \cite{SEMIPARAMTERIC_2016}.

We consider the two datasets used in \cite{SEMIPARAMTERIC_2016}, corresponding to different trajectories of the end-effector. In the first one (\emph{XY-dataset}) the end-effector tracks circles in the XY  plane  of radius $10 cm$  at an approximate speed of $6m/s$; in the second one (\emph{XZ-dataset}), the end-effector tracks similar circles but in the XZ plane (the Z axis corresponds to the vertical direction, parallel to the gravity force). The two circles are tracked using the Cartesian controller proposed in \cite{pattacini2010experimental}. Each dataset contains approximately 8 minutes of data collected  at a sampling rate of $20Hz$, for a total of 10000 points per dataset. One single circle is completed by the robot in about $1.25$ seconds, which corresponds to 25 points.

We shall consider  the models described in Section \ref{sec:semiparametric} endowed with standard input locations or derivative-free features and hyperparameters estimated via either the marginal likelihood approach or Cross Validation\footnote{As discussed in Section \ref{subsec:hyperparameter_estimation}, using the Cross Validation methods is unfeasible when the number of hyperparameters is large; therefore we have not applied validation to the semiparametric model with RBD mean when the mean is to be considered as an hyperparameter nor to the semiparametric model with RBD kernel which has the extra parameter $\gamma^2$.}, the latter has been discussed in \cite{SEMIPARAMTERIC_2016}. For ease of exposition the models will be denoted with a combination of the following shorthands: 
\begin{itemize}
\item
\emph{P}, \emph{NP}, \emph{SP}, \emph{SP2}, \emph{SPK}:  to indicate the model class, as discussed in Section \ref{sec:semiparametric}, approximated according to Section \ref{sec:model_approx/kernel_approx}.
\item \emph{ML}, \emph{CV}: to indicate the method used to estimate the hyperparameters, according to Section \ref{subsec:hyperparameter_estimation}.
\item \emph{DF}, \emph{DFW}, \emph{DFR}, \emph{DFSR}: to indicate the different derivative-free features, according to Section \ref{sec:derivative-free}. If nothing is indicated, the standard input locations are considered.
\end{itemize}

For example, we shall denote ``NP-ML'', the nonparametric model \eqref{nonpar_model} with hyperparameters estimated through maximization of the marginal likelihood; instead 
``NP-ML-DFR'' shall denote the nonparametric model \eqref{eq:derivative_free_model} with hyperparameters estimated through maximization of the marginal likelihood and derivative-free features defined in \eqref{eq:input_locations_features_free}.

The estimation routine has been implemented using Matlab. The RBD regressor $\psi$ for the iCub's right arm has been computed using the  library iDynTree, \cite{nori2015icub}. The Marginal Likelihood has been optimized using the Matlab \verb!fmincon.m! function.
The recursive least squares algorithms have been implemented using the GURLS library, \cite{tacchetti2013gurls}. 
The results of all CV methods are obtained using  code which has been  kindly provided by the authors of \cite{SEMIPARAMTERIC_2016}.

For each model as above, the following online learning scenario is considered (with reference to the general model structure \eqref{modello_approx}):
\begin{itemize}
\item \emph{Initialization}: The first 1000 points in XY-dataset are used to estimate the hyperparameters with one of the two techniques considered, say $\hat{\eta}_{ML}$ and $\hat{\eta}_{CV}$, as well as to compute an initial estimate of the parameter $\theta$, say $\hat{\theta}_{1000}$.
 \item \emph{Online Estimation - Stage 1}:  The remaining 9000 points of the XY-dataset are used to update online the parameter $\theta$ using the recursive least squares algorithm, thus obtaining   $\hat \theta_t$, $t=1001,\dots,10000$. Let  $\hat \theta_F^1 = \hat\theta_{10000}$ be the final value obtained by this procedure on the XY-dataset.

\item \emph{Online Estimation - Stage 2}: The  XZ-dataset is split in 5 sequential subsets (\emph{XZ-dataset-i}, $i=1,..,5$) of 2000 points each (approximately $100$ seconds). These subsets are used to evaluate the performance of the online estimators. For each subset,  the estimator of $\theta$ is always  initialized with $\hat \theta_F^1$ and updated recursively on the $2000$ data of each dataset \emph{XZ-dataset-i}, $i=1,..,5$.  \end{itemize}
Note that the estimators in Stage 2 are initialised using the final estimate from Stage 1, which corresponds to a different motion (XY-dataset).
 Therefore, the evaluation of the performance in Stage 2 allows us to verify how well the estimators  generalise on new unseen data (initial part of the new dataset) as well as how well they are able to learn adapting to a new experimental condition (transient and steady state).

In order to measure the quality of the estimated models, we evaluate the online prediction capability of the estimated models using the following index: 
\begin{align}
\label{eq:pred_error}
\varepsilon^{(k)}_t &= {\frac{\sum_{s=1}^{T}(y^{(k)}({t+s}) -\hat{y}^{(k)}({t+s|t}) )^2}{\sum_{s=1}^{T} (y^{(k)}({t+s}))^2}}
\end{align}
where $\hat{y}^{(k)}(t+s|t)$ is the estimate of  the $k$-th output  $y^{(k)}({t+s})$ at time $t+s$ using the model $\hat{\cal M}_t$ estimated with data up to time $t$. Note that the test data $y^{(k)}({t+s})$, $s=1\ldots T$, have not been used to estimate the model $\hat{\cal M}_t$,
Thus, $\varepsilon^{(k)}(t)$ is an average relative error on a future horizon of length $T$ for the output component $k$. In addition, $\varepsilon^F(t)$ and $\varepsilon^T(t)$ will be the average values of $\varepsilon^{(k)}(t)$ for the 3 forces ($k=1,2,3$) and the 3 torques ($k=4,5,6$), respectively. In all our simulations, $T$ will be equal to $25$ (which roughly corresponds to one revolution of the end effector in the $XZ$ plane).

The index \eqref{eq:pred_error} can be motivated, for instance, by the  possible use of the model in the framework of  model predictive control \cite{maciejowski2002predictive}, where the final performance of the controller hinges on the predictive capabilities of the model  on a given future horizon (say $T$). 
 
The experimental results will be first presented for models that use numerical derivatives and then for the derivative-free models; a comparison will be eventually provided between the two. For each experiment, we show $\varepsilon^F(t)$ and $\varepsilon^T(t)$ averaged over the $5$ subsets XZ-Dataset-i, $i =1,..,5$.
\subsection{Experimental results using numerical derivatives}
\label{sec:simulations_stand_inp_loc}

In Figure \ref{fig:prediction_error}
\begin{figure}[hbtp]
\centering
\includegraphics[width=1\columnwidth]{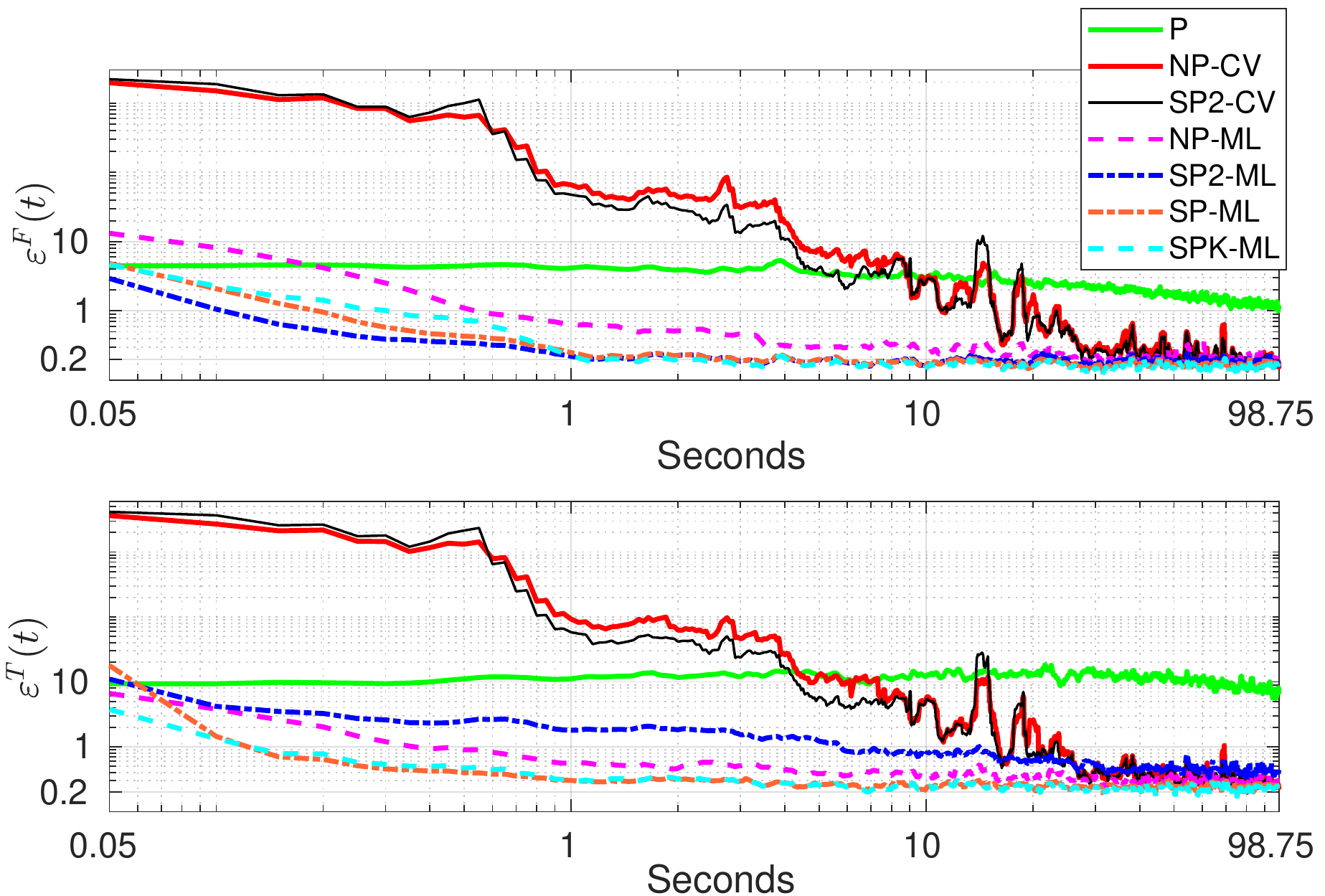}
\caption{Average (over the 5 subsets of 100 seconds each) of the relative squared prediction errors $\varepsilon^F(t)$ and $\varepsilon^T(t)$, computed with $T=25$ corresponding to an  horizon of 1.25 seconds. A  log-log scale is used for ease of readability.}
\label{fig:prediction_error}
\end{figure}
the behavior of $\varepsilon^F(t)$ and $\varepsilon^T(t)$  is presented. The parametric model, P, exhibits a poor performance because it describes only crude idealizations of the actual dynamics. The algorithms based on Cross Validation (CV)  perform significantly worse in the first $60$ seconds than those based on Marginal Likelihood (ML) optimisation; this is not unexpected as discussed in \cite{Pillonetto2015106}.

As expected, the nonparametric model, NP-ML, has worse generalization performance (the error is larger in the first few steps) but better adaptation  capabilities with respect to model P. The models with the best performance are SP-ML and SPK-ML because they combine the benefit of the parametric approach, i.e. generalization capabilities (good estimation performance at the beginning of the new dataset) and of the nonparametric approach, i.e. learning capabilities (good transient and steady state performance). The estimator SP2-ML should partially inherit these benefits from the SP structure, yet it shows an overall  slightly worse performance. This is due to the fact that  the first (least squares) step, i.e. the estimation of the linear model, is subject to a strong bias deriving from the unmodeled dynamics. Instead, a sound approach is followed by SP-ML and SPK-ML  in which the estimation of the hyperparameters is performed  jointly, avoiding such bias. In the steady state these semiparametric models outperform the others; yet the semiparametric SPK-ML performs  best both in terms of average as well as distribution, as shown in  Figure \ref{fig:prediction_error_boxplot}
\begin{figure}[hbtp]
\centering
\includegraphics[width=\columnwidth]{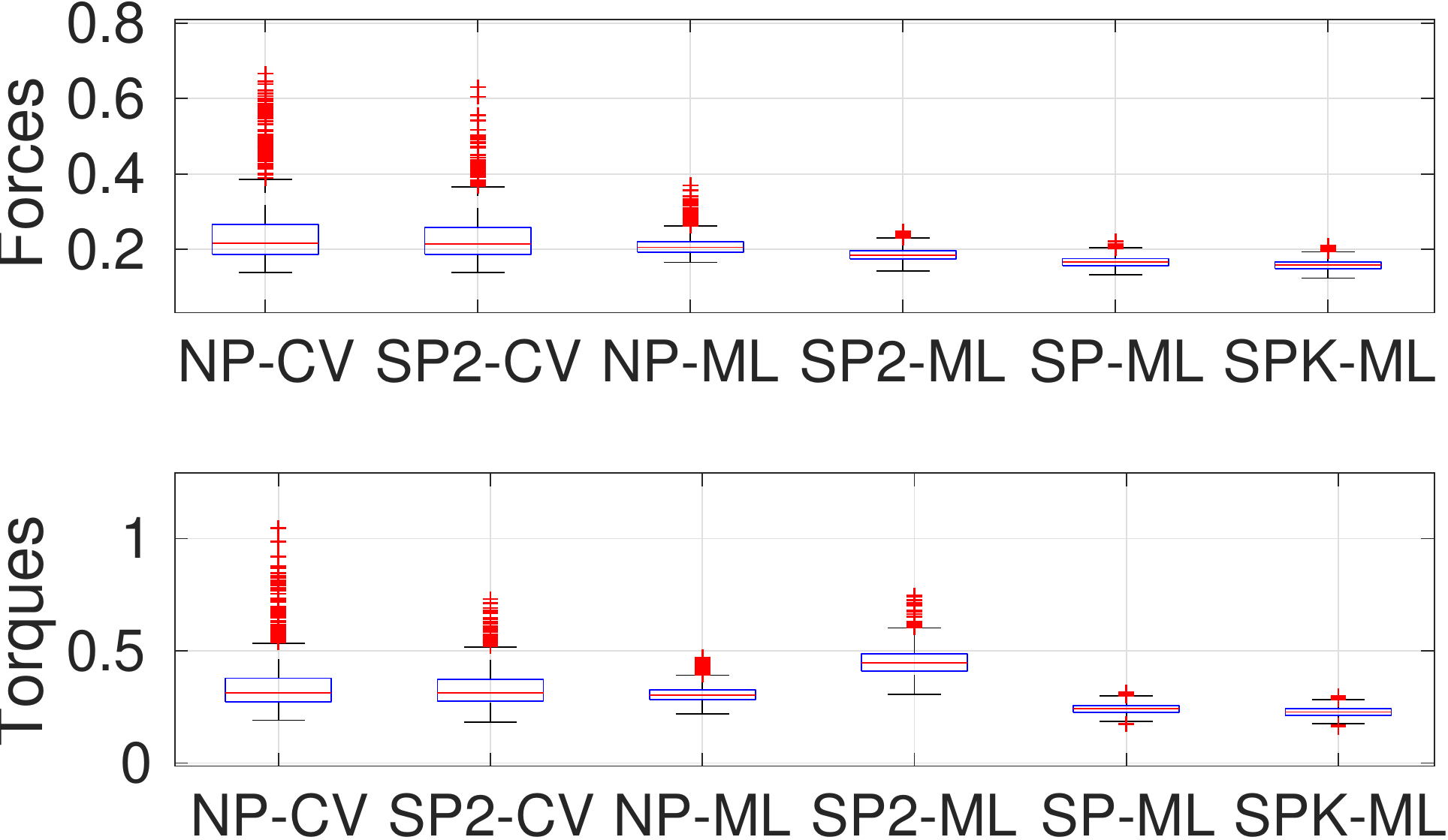}
\caption{Boxplots of the steady state (i.e.  after $30$ seconds, see also Figure \ref{fig:prediction_error}) of the relative squared prediction errors 
$\varepsilon^F(t)$ and $\varepsilon^T(t)$, computed with $T=25$ corresponding to an  horizon of 1.25 seconds.}
\label{fig:prediction_error_boxplot}
\end{figure}
which reports the boxplots of $\varepsilon_t^F$ and $\varepsilon_t^T$ in ``steady state'', i.e. after the first $30$ seconds which is considered to be transient, see Figure  \ref{fig:prediction_error}. The reader is also referred to  the discussion in Appendix \ref{app:discussion} for a theoretical justification of this latter fact.

\subsection{Experimental results with derivative-free features}

The focus of this section is the analysis of the derivative-free models, including the comparison with the  models that use numerical derivatives  presented in Section \ref{sec:simulations_stand_inp_loc}. 


The parameters $M$ (number of past temporal lags used to form the features vector $\xi(t)$, see equations  \eqref{eqRdiag},  \eqref{eq:input_locations_features_free} and \eqref{eq:input_locations_features_vel_acc}, and the number of features, $k$, in the reduced derivative-free models are set as follows: 
\begin{itemize}
\item $M$ is fixed equal to $10$; larger values have been tested with no significant differences in performance. 
\item $k$ is fixed to $3$. A discussion on this choice as well as some results with different choices can be found later in this Section, see e.g. Figure \ref{fig:prediction_error_boxplot_FSF_2_3_4}.
\end{itemize}

 The hyperparameters are estimated using marginal likelihood (ML)  maximization, since performing Cross Validation via gridding in high dimension is computationally unfeasible. 

In the first comparison, the nonparametric methods (NP-ML) is compared to its derivative-free versions NP-ML-DF, NP-ML-DFW, NP-ML-DFR and NP-ML-DFSR.

\begin{figure}[hbtp]
\centering
\includegraphics[width=\columnwidth]{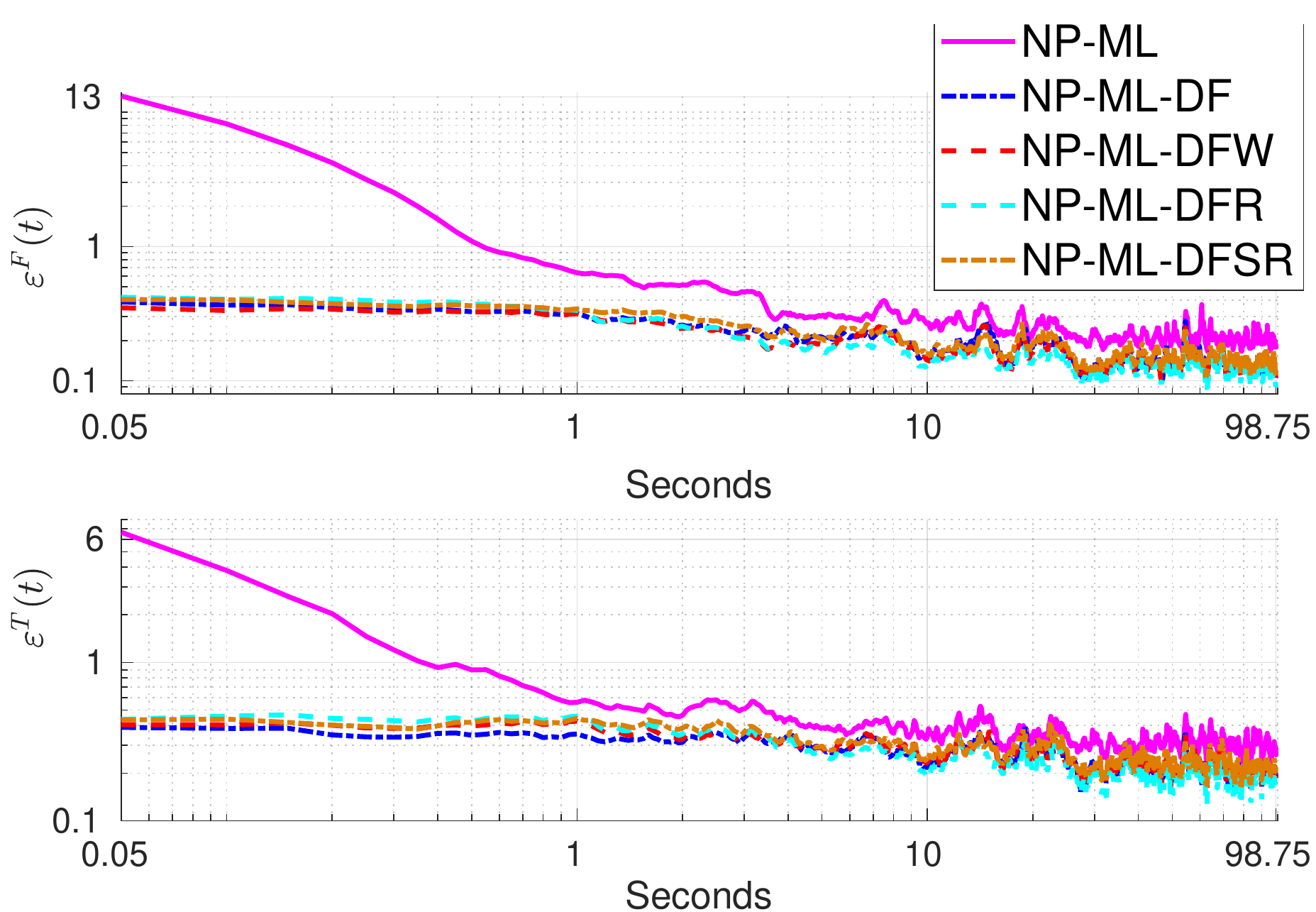}
\caption{Average (over the 5 subsets of 100 seconds each) of the relative squared prediction errors $\varepsilon^F(t)$ and $\varepsilon^T(t)$, computed with $T=25$ corresponding to a  horizon of 1.25 seconds. A  log-log scale is used for ease of readability.}
\label{fig:prediction_error_stand_FSF_FVA}
\end{figure}

In Figure \ref{fig:prediction_error_stand_FSF_FVA} the averaged (over $5$ realisations) time evolutions of  $\varepsilon^F(t)$ and $\varepsilon^T(t)$ are illustrated. All the nonparametric derivative-free  models perform comparably and outperform NP-ML, both in transient (more significant) as well as in steady state. The distribution of the steady state errors is shown using boxplots in Figure \ref{fig:prediction_error_boxplot_stand_FSF_FVA}. It is clear that all the derivative-free methods outperform NP-ML, but also that NP-ML-DFR (which uses a reduced rank derivative-free feature) is the best performing method.  This confirms that the dimensionality reduction in equation \eqref{eq:input_locations_derivativefree_general} captures the relevant information and allows to  reduce the variance of the estimators. 

\begin{figure}[hbtp]
\centering
\includegraphics[width=\columnwidth]{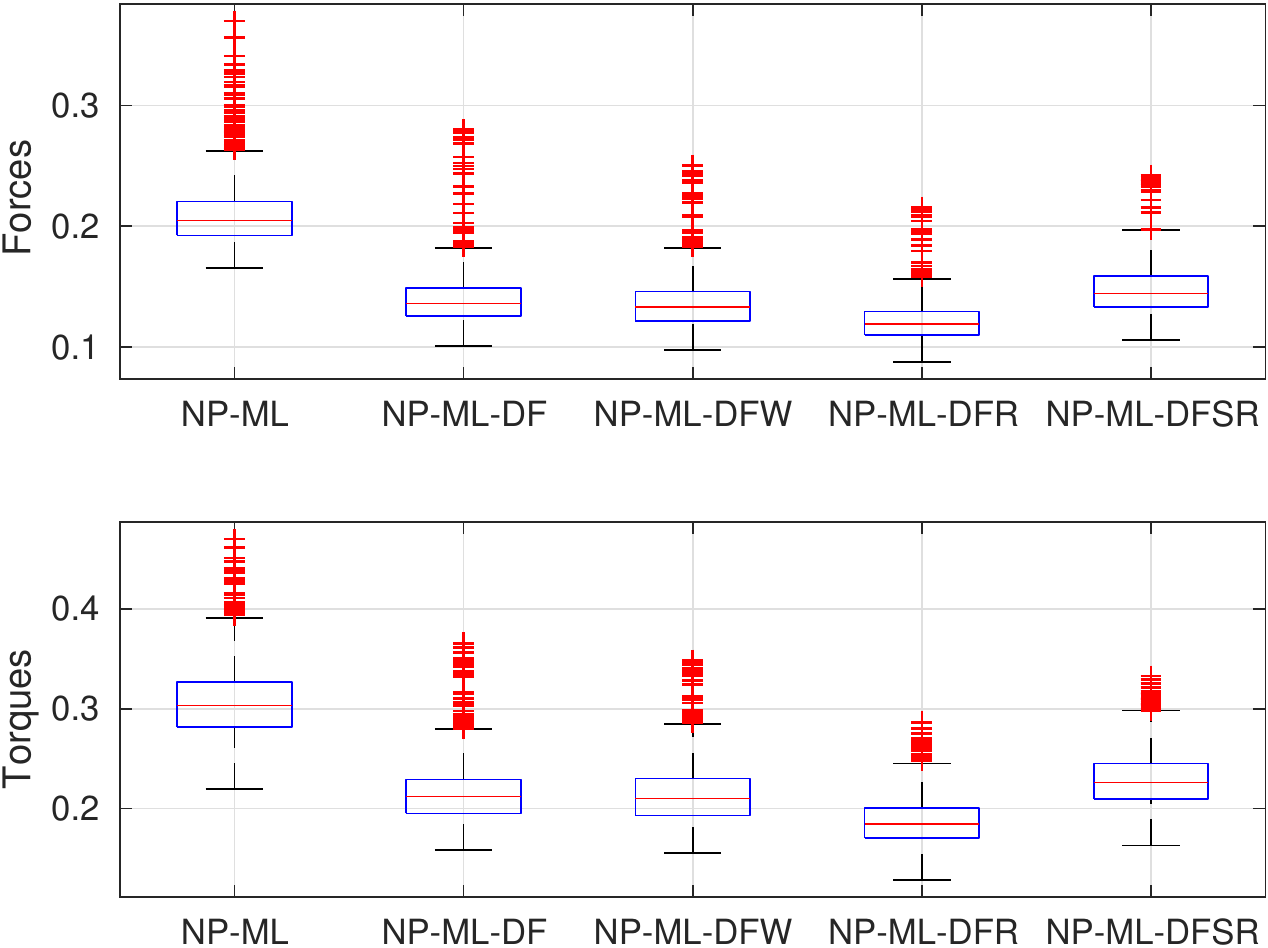}
\caption{Boxplots of the steady state (i.e.  after $30$ seconds, see also Figure \ref{fig:prediction_error_stand_FSF_FVA}) of the relative squared prediction errors 
$\varepsilon^F(t)$ and $\varepsilon^T(t)$, computed with $T=25$ corresponding to a  horizon of 1.25 seconds.}
\label{fig:prediction_error_boxplot_stand_FSF_FVA}
\end{figure} 

As anticipated, the results in Figure \ref{fig:prediction_error_stand_FSF_FVA} and Figure \ref{fig:prediction_error_boxplot_stand_FSF_FVA} are obtained setting $k=3$ in NP-ML-DSR. The reasons for this choice is that, when using numerical differentiation, the input location vector $x(t)$ contains exactly 3 components (position, velocity and acceleration) for each DoF. It is natural to ask what happens as $k$ changes. In Figure  \ref{fig:prediction_error_boxplot_FSF_2_3_4} the steady state behaviour\footnote{The transient behaviour has been omitted because it does not add much information with respect to the steady state statistics.} of $\varepsilon^F(t)$ and of $\varepsilon^T(t)$ is analysed as a function of $k=1,2,3,4$.

It is apparent from Figure \ref{fig:prediction_error_boxplot_FSF_2_3_4} that $k=1$ and $k=2$ are not sufficient and lead to a considerable performance degradation; in addition, no improvement is obtained  increasing $k$ beyond $3$ (compare $k=3$ and $k=4$ in Figure \ref{fig:prediction_error_boxplot_FSF_2_3_4}). Larger values (e.g. $k=5$) have also been tested leading to similar results and have been therefore omitted. 
The results in Figure  \ref{fig:prediction_error_boxplot_stand_FSF_FVA} show that the performance of NP-ML-DF (which uses $R=I$)
 is worse than NP-ML-DFR with $k=3$. Overall this suggests that, indeed, $k=3$ is the optimal choice for this specific application.
 NP-ML-DFR with $k=3$ performs better than NP-ML-DFSR. Considering the bias-variance trade-off dilemma, the latter method introduces more bias and less variance than the former, suggesting that, in these experimental results, the bias introduced by NP-ML-DFSR is preponderant with respect to the reduction of the variance.

\begin{figure}[hbtp]
\centering%
\includegraphics[width=\columnwidth]{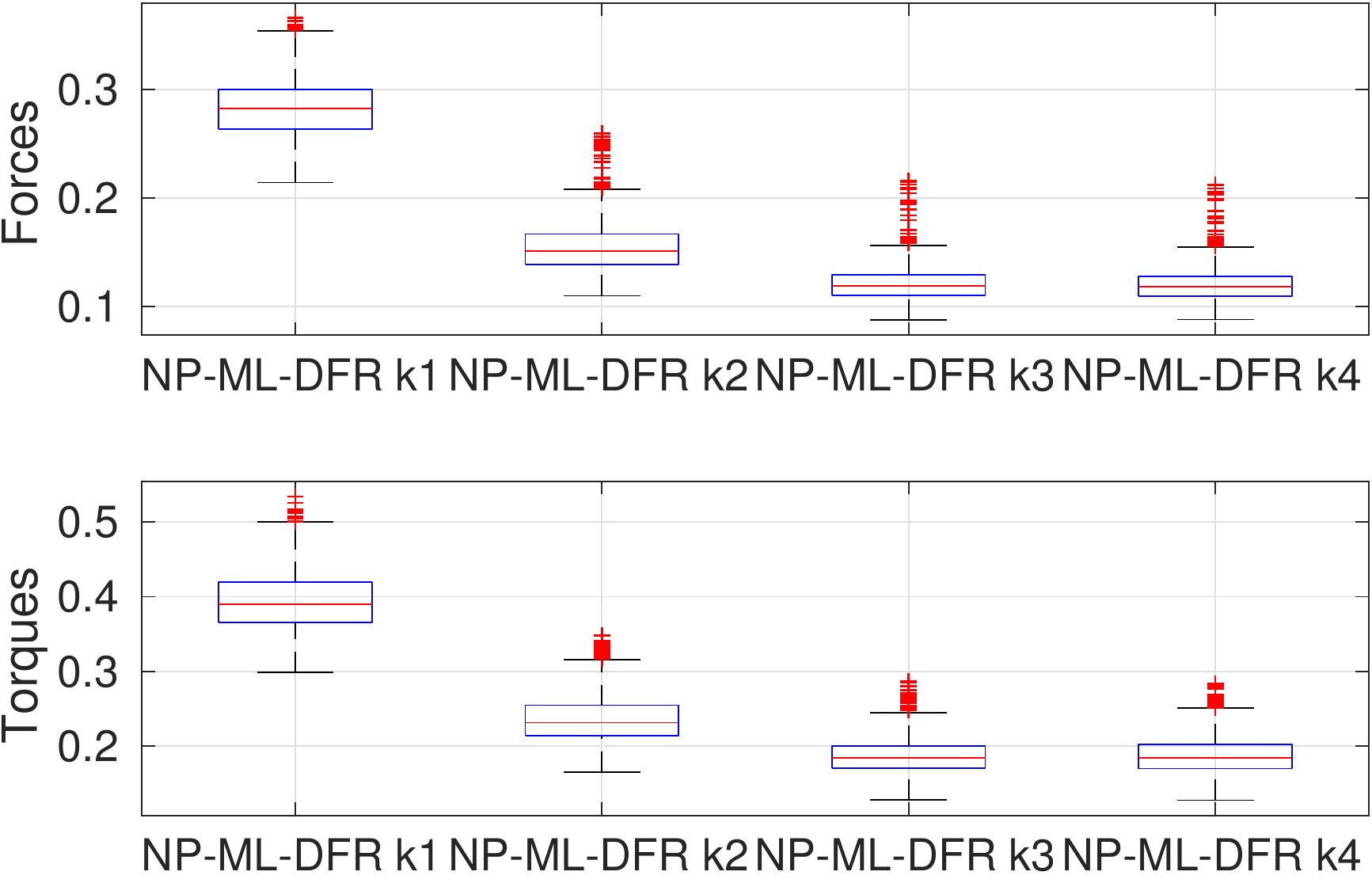}
\caption{Boxplots of the steady state (i.e. after $30$ seconds) of the relative squared prediction errors 
$\varepsilon^F(t)$ and $\varepsilon^T(t)$, computed with $T=25$ corresponding to a  horizon of 1.25 seconds.}
\label{fig:prediction_error_boxplot_FSF_2_3_4}
\end{figure}

Finally,  a  comparison between the estimators using numerical derivatives and the derivative-free methods (and in particular DFR) is provided for the semiparametric models.

In Figure \ref{fig:prediction_error_SP_comparison} and Figure \ref{fig:prediction_error_boxplot_SP_comparison} the behaviour of $\varepsilon^F(t)$ and of $\varepsilon^T(t)$ for the models SPK-ML-DFR and SPK-ML is illustrated.
\begin{figure}[hbtp]
\centering
\includegraphics[width=\columnwidth]{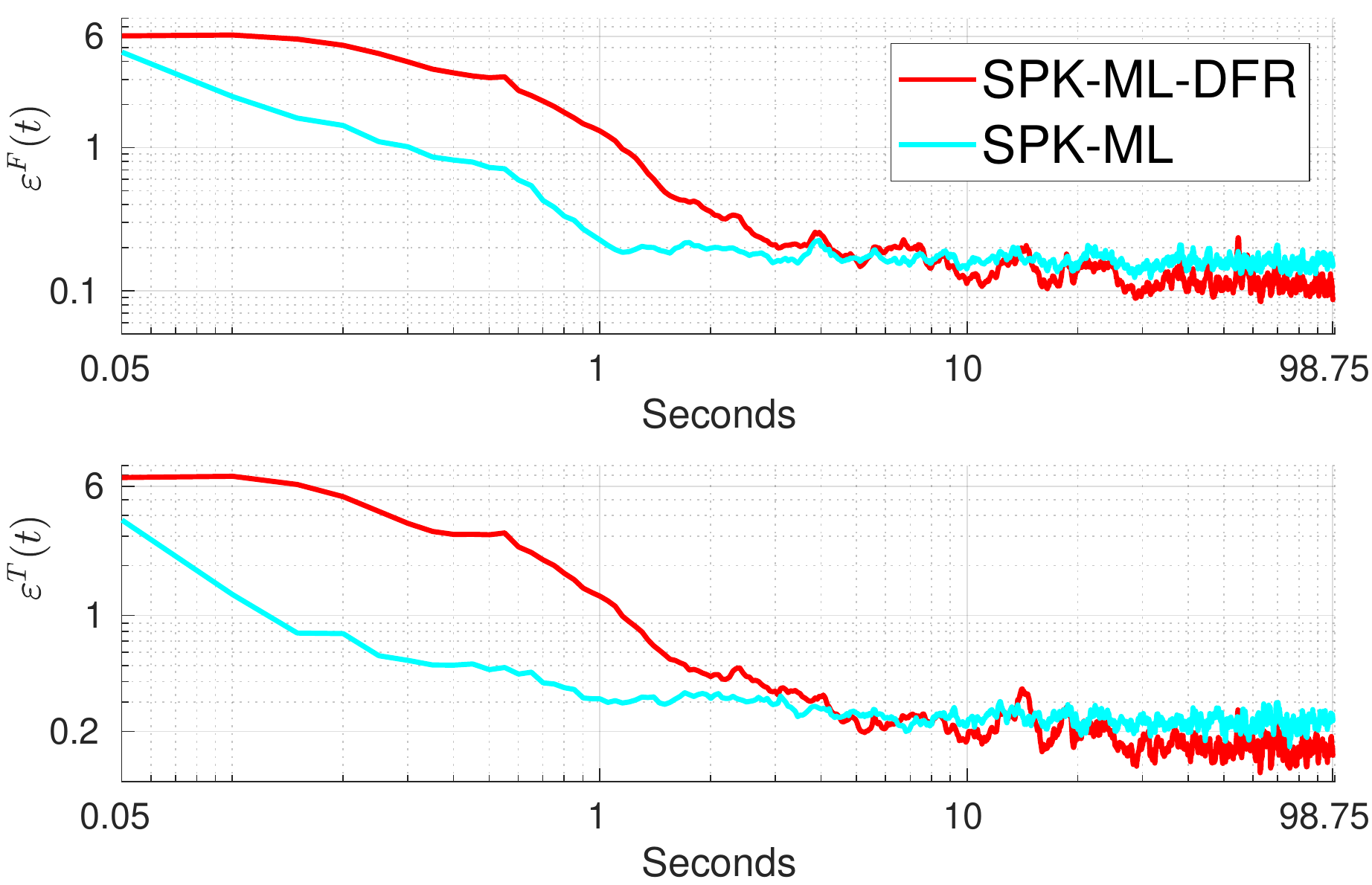} 
\caption{Average (over the 5 subsets of 100 seconds each) of the relative squared prediction errors $\varepsilon^F(t)$ and $\varepsilon^T(t)$, computed with $T=25$ corresponding to a  horizon of 1.25 seconds. A  log-log scale is used for ease of readability.}
\label{fig:prediction_error_SP_comparison}
\end{figure}

\begin{figure}[hbtp]
\centering%
\includegraphics[width=\columnwidth]{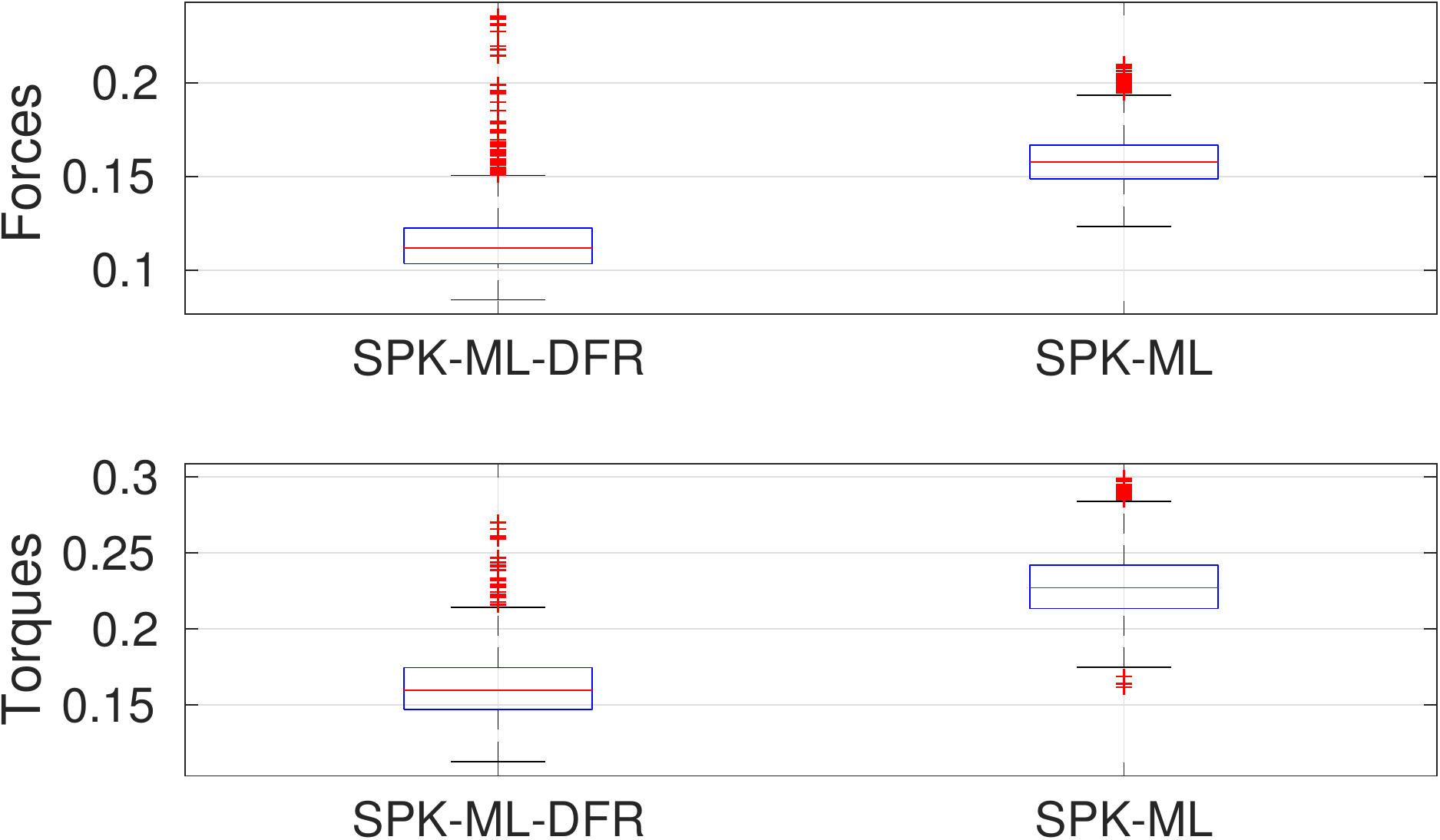}
\caption{Boxplots of the steady state (i.e.  after $30$ seconds, see also Figure \ref{fig:prediction_error_SP_comparison}) of the relative squared prediction errors 
$\varepsilon^F(t)$ and $\varepsilon^T(t)$, computed with $T=25$ corresponding to a  horizon of 1.25 seconds.}
\label{fig:prediction_error_boxplot_SP_comparison}
\end{figure}

The two semiparametric methods have  similar initial performance, but SPK-ML has a a better learning rate in the transient. This behaviour  might  be attributed to the parametric component of the model since in SPK-ML-DFR the physical meaning of the features is lost. However, in steady state the semiparametric derivative-free model outperforms the standard model with an improvement of about $30\%$ in terms of  relative error, see Figure \ref{fig:prediction_error_boxplot_SP_comparison}.

The results obtained in this section give an empirical evidence that learning the features directly from the past history of the measured joint trajectories is a rather promising direction.

\subsection{Comparison among DFR-like models}
Finally a comparison among the models with DFR features, namely, P-ML-DFR, NP-ML-DFR, SP2-ML-DFR and SPK-ML-DFR, is presented in Figures \ref{fig:prediction_error_FSF_all} and \ref{fig:prediction_error_boxplot_FSF_all}.

\begin{figure}[hbtp]
\centering
\includegraphics[width=\columnwidth]{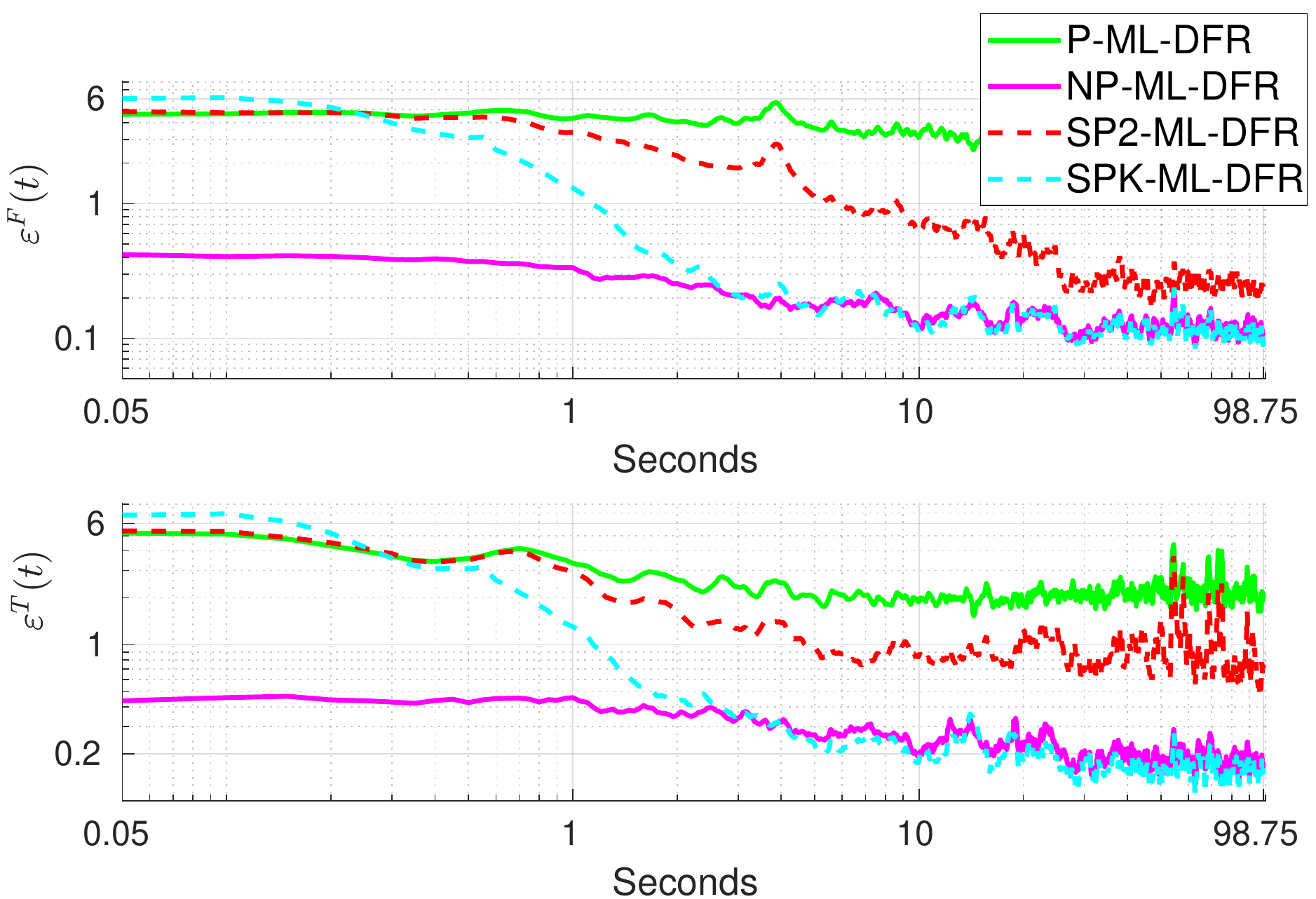}
\caption{Average (over the 5 subsets of 100 seconds each) of the relative squared prediction errors $\varepsilon^F(t)$ and $\varepsilon^T(t)$, computed with $T=25$ corresponding to a  horizon of 1.25 seconds. A  log-log scale is used for ease of readability.\label{fig:prediction_error_FSF_all}}
\end{figure}

\begin{figure}[hbtp]
\centering%
\includegraphics[width=\columnwidth]{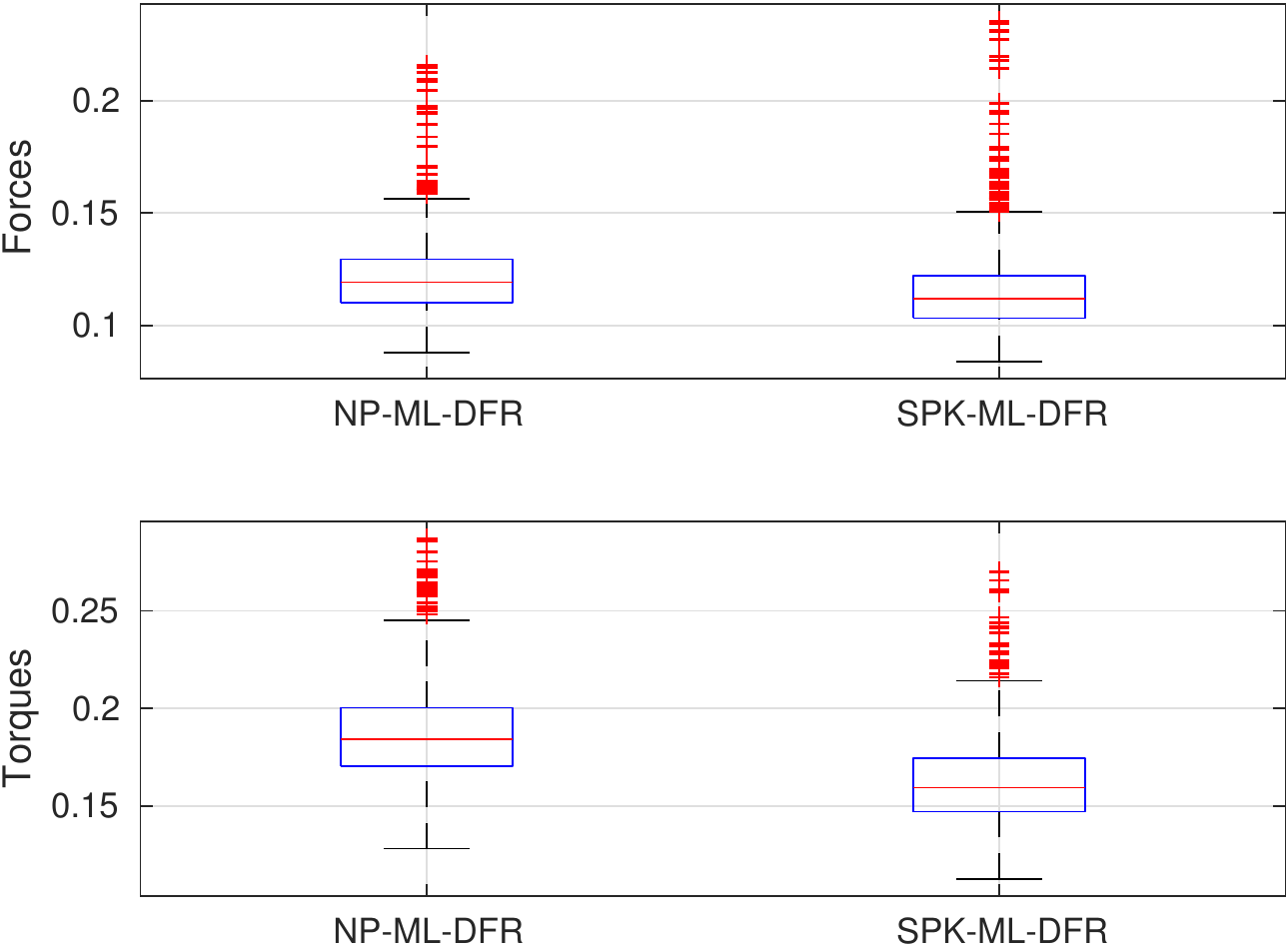}
\caption{Boxplots of the steady state (i.e. after $30$ seconds, see also Figure \ref{fig:prediction_error_FSF_all}) of the relative squared prediction errors 
$\varepsilon^F(t)$ and $\varepsilon^T(t)$, computed with $T=25$ corresponding to a  horizon of 1.25 seconds.}
\label{fig:prediction_error_boxplot_FSF_all}
\end{figure}


The transient performance can be analyzed by observing the first 30 seconds in Figure \ref{fig:prediction_error_FSF_all}. Similarly to P-ML, also the parametric model P-ML-DFR has very poor performance (note that the relative error is always larger than 1); the poor performance of the parametric models seems to negatively affect also the transient behaviour of the semiparametric models. As a matter of fact, the transient performances of SP2-ML-DFR and  SPK-ML-DFR are very poor. 
The nonparametric model, NP-ML-DFR, definitely outperforms all the other models in terms of transient, suggesting that the parametric models are probably  inadequate
to capture the dynamic behaviour. 

The last 60 seconds of the experiment reported in  Figure \ref{fig:prediction_error_FSF_all} and the boxplots in  Figure \ref{fig:prediction_error_boxplot_FSF_all} (only for NP-ML-DFR and SPK-ML-DFR) illustrate the behaviour at steady state. The suboptimal model SP2-ML-DFR is largely  unsatisfactory, which can be probably attributed to the bias error introduced by the parametric component estimated by least squares. 

\subsection{Experimental results discussion}

We can summarise as follows the findings of our extensive experimental study: 
\begin{itemize}
\item The derivative-free methods outperform the schemes based on numerical differentiation; this is even more remarkable if one recalls that the numerical derivatives have been computed starting from data at a higher sampling rate, i.e. they have practically used a richer dataset (which ideally should have lead to  better noise-rejection properties). The main reason is probably to be attributed to the fact that the derivatives on which the ``classical'' models  rely are computed using numerical differentiation schemes from the measured positions, which are subject to noise. However, the models \eqref{nonpar_model}, \eqref{eq:SP-Mean_model_statistics} and \eqref{modelloSPK} do not account for this noise and operate as if the measurements were correct. 
Instead, the derivative free method performs a reduction, which can be assimilated to computing derivatives and accelerations, but this reduction is computed as part of the modelling exercise, and as such its effect on wrench prediction is directly accounted for.
\item The transient performances of semiparametric models strongly depend on the availability of an accurate (physical) model. Indeed, when using derivative-free methods wherein the  extracted features are not  guaranteed to approximate velocities and accelerations, the parametric model loses physical meaning. This is confirmed by the poor transient performance of  P-ML-DFR, SP2-ML-DFR and SPK-ML-DFR (see Figure \ref{fig:prediction_error_FSF_all}).
\item Depending on whether transient or steady state performance is to be favoured, NP-ML-DFR or SPK-ML-DFR  should be preferred respectively (see Figures \ref{fig:prediction_error_FSF_all}, \ref{fig:prediction_error_boxplot_FSF_all}).
\end{itemize}


\section{Conclusions}
\label{sec:conclusions}

In this paper, several models and algorithms for online inverse dynamics learning have been discussed and compared in a common framework. Parametric and nonparametric methods have been considered, as well as semiparametric models obtained by combining the two. Different strategies, Cross Validation (CV) and Marginal Likelihood optimisation (ML),  for estimating the hyperparameters have been compared and also new model structures which do not rely on pre-computed numerical derivatives have been introduced.

All these algorithms have been thoroughly tested on real data from the right arm of the iCub robot; the comparison has been performed both in terms of transient behaviour (how fast algorithms can adapt to a new experimental setup) as well as steady state behaviour. 

Overall our experimental validation suggests that:
\begin{itemize}
\item The non-parametric model (NP) and the Semiparametric model with RBD mean (SP2) perform better when trained using Marginal Likelihood optimisation rather than Cross Validation algorithms available in \cite{SEMIPARAMTERIC_2016}; \item Semiparametric methods which exploit physical insight, when using numerical derivatives,  outperform purely nonparametric structures;
\item Derivative-free methods are definitely advantageous w.r.t. ad-hoc methods which rely on numerical differentiation, when applied to purely non-parametric model structures (see e.g. NP-ML-DFR);
\item Endowing semiparametric methods with derivative-free schemes is not entirely trivial; these new features do not yield improvements in the parametric model as much as they do in the nonparametric one. As a result the transient behaviour of SPK-ML-DFR is significantly worse than its nonparametric counterpart; instead, SPK-ML-DFR outperforms its nonparametric counterpart in steady state.
\end{itemize}

This last item  calls for further research efforts in the future, which would hopefully allow to fully exploit the benefits of derivative-free methods coupled with semiparametric model structures.  Our future agenda includes further comparison with parametric methods which account for physical consistency of the parameters (see for instance \cite{TraversaroBEN16}) as well as the implementation of control strategies which exploit  the estimated models.  
  
\section{ACKNOWLEDGMENTS}
The authors gratefully acknowledge the other teams involved in the FIRB project ``Learning meets time'' (RBFR12M3AC), in particular  iCub Facility group (Francesco Nori, Giorgio Metta)  and Lorenzo Rosasco (LCSL-IIT@MIT) for making their data and code available to us.


\appendix\label{AppA}
\subsection{Proof of Proposition \ref{proposition:connection_SPK_SPmean}}
\label{app:proof}

First, note that both the SP model and the SPK model can be rewritten as
\al{\label{model_proof} y(t)=\psi^\top(x(t))\pi +f(x(t))+e(t)}
where :
\begin{enumerate}
\item $\pi$ is an unknown but fixed quantity for the SP model;
\item $\pi$ is a zero mean Gaussian random vector with covariance matrix $\gamma^2 I_p$ for the SPK model.
\end{enumerate} In both models, $f(\cdot)$ is a Gaussian process with kernel function
\al{K(x(t),x(s))=\rho^2 K_G(x(t),x(s)).\nn}

A well known connection between Bayes and Fisher (i.e. assuming the  parameter $\pi$ is an unknown but fixed quantity) estimators, is that the latter can be obtained as a limiting case of the former when:
\begin{itemize}
\item the parameter $\pi$ is modeled as a zero mean Gaussian vector with covariance matrix $\gamma^2 I_p$
\item the variance of the prior distribution of $\pi$ is let go to infinity  by letting $\gamma^2\rightarrow \infty$.
\end{itemize}
This proves the first part of the Proposition. We proceed to prove the second part.  Let us stack the available data $y(t)$, $t=1\ldots N$ in the vector $\mathbf{y}$ and stack correspondingly the regressors $\psi^\top(x(t)) $ and $f(x(t))$ in the matrix $\Psi$ and vector $\mathbf{f}$, respectively, so that models \eqref{model_proof} can be written as 
\begin{equation}\label{Vec:model}
\mathbf{y} = \Psi \pi + \mathbf{f} + \mathbf{e}
\end{equation}
where $\mathbf{e}$ is defined with the same rule as $\mathbf{y}$. Moreover, we define $\mathbf{K}(\mathbf x,\mathbf x)=\mathrm{cov}[\mathbf f,\mathbf f]$ and $\mathbf{K}(\cdot, \mathbf x )=\mathrm{cov}[f(\cdot),\mathbf f]$.
The minimum variance estimators of $\pi$ and $f$ under 2) and given $\rho^2,\sigma^2,\tau$, are:
\al{
 \hat\pi &= \mathrm{cov}[\pi,\mathbf{y}] \mathrm{Var}^{-1}[{\mathbf{y}}] \mathbf{y} \nn \\ & =   \gamma^2 \Psi^\top \left(\gamma^2 \Psi \Psi^\top + \mathbf{K}(\mathbf x,\mathbf x) + \sigma^2 I\right)^{-1} \mathbf{y} \label{Bayes}  \\
\hat f(\cdot) &= \mathrm{cov}[f(\cdot),\mathbf{y}] \mathrm{Var}^{-1}[\mathbf{y}] \mathbf{y} \nn \\ &  =   \mathbf{K}(\cdot, \mathbf x ) \left(\gamma^2 \Psi \Psi^\top + \mathbf{K}(\mathbf x,\mathbf x) + \sigma^2 I\right)^{-1} \mathbf{y}. \nn
}
Defining $R:=\mathbf{K}(\mathbf x,\mathbf x) + \sigma^2 I$ and using the matrix inversion lemma we have:
\al{
& \left(\gamma^2  \Psi \Psi^\top + \mathbf{K}(\mathbf x,\mathbf x) + \sigma^2 I\right)^{-1} \nn\\ & \hspace{0.5cm}=  \left(\gamma^2 \Psi \Psi^\top + R\right)^{-1}\nn \\
&\hspace{0.5cm} =R^{-1} - R^{-1} \Psi \left(\Psi^\top R^{-1} \Psi + \gamma^{-2}I\right)^{-1} \Psi^\top R^{-1}
\nn
}
so that, from \eqref{Bayes}
\al{
\hat\pi &= \gamma^2 \left(I   -  \Psi^\top R^{-1} \Psi \left(\Psi^\top R^{-1} \Psi + \gamma^{-2}I\right)^{-1} \right)\Psi^\top R^{-1} \mathbf{y} \nn\\
& =  \left(\Psi^\top R^{-1} \Psi + \gamma^{-2}I\right)^{-1}\Psi^\top R^{-1} \mathbf{y}  \nn
,}
and similarly
\al{
\hat f(\cdot) &= \mathbf{K}(\cdot,\mathbf x) R^{-1} \left(I   -    \Psi \left(\Psi^\top R^{-1} \Psi + \gamma^{-2}I\right)^{-1} \Psi^\top R^{-1} \right)  \mathbf{y} \nn\\
& =  \mathbf{K}(\cdot,\mathbf x)  R^{-1} [\mathbf{y} - \Psi \hat \pi] . 
\nn} Clearly, as $\gamma^2 \rightarrow \infty$, we have that $\hat\pi$ converges to the \emph{weighted} least squares estimate
\begin{equation}\label{WLS}
 \hat \pi_{WLS} =\left(\Psi^\top R^{-1} \Psi \right)^{-1}\Psi^\top R^{-1} \mathbf{y} 
\end{equation}
and $\hat f$ converges to 
$$
\hat f(\cdot)  =\mathbf{K}(\cdot,\mathbf x)  R^{-1} [\mathbf{y} - \Psi \hat \pi_{WLS}] . 
$$

On the other hand the marginal  likelihood function for model \eqref{Vec:model}, under 1), i.e.  when $\pi$ is considered as an unknown parameter, has the form:
$$
\begin{array}{rcl}
L_{SP}(\mathbf{y}) & = & -2 {\rm log}\,(p_\eta(\mathbf{y}))  \\
& = &{\rm log}({\rm det}(2\pi R)) + (\mathbf{y}-\Psi\pi)^\top R^{-1} (\mathbf{y}-\Psi\pi).
\end{array}
$$
When $\rho^2,\sigma^2,\tau$ are kept fixed, the minimization with respect to $\pi$ can be performed in closed form, and yields exactly the weighted least squares solution
\eqref{WLS}.  However, even for $\gamma^2 \rightarrow \infty$, the marginal likelihoods of $\mathbf{y}$ given the hyperparameters $\rho,\sigma^2,\tau$ under 1) and 2) are different. In fact, if  1) is postulated, and $\pi$ is solved as above,  one obtains the  \emph{profile} marginal log-likelihood $\hat L_{SP}(\mathbf{y}) := L_{SP}(\mathbf{y})_{|  \pi= \hat\pi_{WLS}}$
\begin{equation}
\hat L_{SP}(\mathbf{y}) = {\rm log}({\rm det}(2\pi R)) + (\mathbf{y}-\Psi \hat \pi_{WLS} )^\top R^{-1} (\mathbf{y}-\Psi \hat \pi_{WLS} ) \nn	
\end{equation}
where  the hyperparameters $\rho^2,\sigma^2,\tau$ are hidden in the definition of $R = \mathbf{K}(\mathbf x,\mathbf x) + \sigma^2 I$.

Instead, if 2) is postulated, the marginal log-likelihood takes the form 
$$
L_{SPK}(\mathbf{y}) = {\rm log}({\rm det}(2\pi(\gamma^2 \Psi\Psi^\top + R))) + \mathbf{y}^\top (\gamma^2 \Psi\Psi^\top + R)^{-1}\mathbf{y}.
$$
Using, as above, the matrix inversion Lemma on $(\gamma^2 \Psi\Psi^\top + R)$, the {\em Sylvester} determinant identity 
and defining $\hat \pi:=(\Psi^\top R^{-1}\Psi + \gamma^{-2}I)^{-1} \Psi R^{-1} \mathbf{y}$, we obtain 
\al{
L_{SPK}(\mathbf{y}) =  {\rm log}&({\rm det}(2\pi R) )+  {\rm log}({\rm det}(I_p + \gamma^2\Psi^\top R^{-1} \Psi)) \nn \\
&  + (\mathbf{y}-\Psi \hat\pi)^\top R^{-1} (\mathbf{y}-\Psi \hat\pi) 	\nn \\
& +  \hat\pi^\top \Psi^\top R^{-1} (\mathbf{y}-\Psi \hat \pi) .\nn
}As $\gamma^2 \rightarrow \infty$ we have that $\hat \pi \rightarrow \hat\pi_{WLS}$ and $  \hat\pi^\top \Psi^\top R^{-1} (\mathbf{y}-\Psi \hat \pi)
\rightarrow 0$ so that\footnote{The symbol $\approx$ denotes equality in the limit as $\gamma \rightarrow\infty$.}
\al{L_{SPK}(\mathbf{y})  \approx  {\rm log}&({\rm det}(2\pi R)) +  {\rm log}({\rm det}(I_p+ \gamma^2\Psi^\top R^{-1} \Psi)) \nn\\
&  + (\mathbf{y}-\Psi \hat\pi_{WLS})^\top R^{-1} (\mathbf{y}-\Psi \hat\pi_{WLS}) . \nn
}
The second term ${\rm log}({\rm det}(I_p + \gamma^2\Psi^\top R^{-1} \Psi)) $ can be manipulated as follows:
\al{
{\rm log}& ({\rm det}(I_p + \gamma^2\Psi^\top R^{-1} \Psi))\nn  \\
&= {\rm log}({\rm det}(\gamma^2 I_p)) +  {\rm log}({\rm det}(\gamma^{-2} I_p + \Psi^\top R^{-1} \Psi))\nn\\
& \approx  p \, {\rm log} \,\gamma^2+  {\rm log}({\rm det}( \Psi^\top R^{-1} \Psi))
 \nn}
 where the last approximation clearly holds when $\gamma^2 \rightarrow\infty$.
Inserting the last expression in $L_{SPK}({\mathbf{y}})$ we obtain that
\begin{equation}\label{ML_K2}
\begin{array}{rcl}
L_{SPK}(\mathbf{y}) &\approx& \!\! {\rm log}({\rm det}(2\pi R)) + p  {\rm log} \gamma^2+  {\rm log}({\rm det}( \Psi^\top R^{-1} \Psi)) \nn\\
& & + (\mathbf{y}-\Psi \hat\pi_{WLS})^\top R^{-1} (\mathbf{y}-\Psi \hat\pi_{WLS})  \nn\\
&\approx &\!\!  \hat L_{SP}(\mathbf{y})  + {\rm log}({\rm det}( \Psi^\top R^{-1} \Psi))  + p \, {\rm log} \,\gamma^2
\end{array}
\end{equation}
which shows that  the the two log-likelihoods differ, up to the constant $ {\rm log}({\rm det}(\gamma^2 I_p)) = p \, {\rm log} \,\gamma^2$ which is not a function of $\rho,\sigma^2,\tau$,  for a nontrivial term ${\rm log}({\rm det}( \Psi^\top R^{-1} \Psi)) $ which have an influence on the location of their minima.  \qed

\bibliographystyle{IEEEtran}
\bibliography{biblio}

\end{document}